%% file: arxiv_version.tex
\documentclass{article}

\usepackage{arxiv}
\input{math_commands.tex}

\usepackage{amsmath}
\usepackage{amssymb}
\usepackage{amsthm}
\usepackage{mathtools}
\usepackage{bm}
\usepackage{enumitem}
\usepackage{microtype}
\usepackage{url}
\usepackage{hyperref}
\usepackage{dsfont}
\usepackage{bbm}
\usepackage{wrapfig}
\usepackage[dvipsnames]{xcolor}
\usepackage{booktabs} 
\newtheorem{theorem}{Theorem}

\theoremstyle{definition}
\newtheorem{definition}{Definition}

\theoremstyle{remark}

\newcommand{\Ccal}{\mathcal{C}}
\newcommand{\Ucal}{\mathcal{U}}

\newcommand{\Risk}{\mathcal{R}}
\newcommand{\ck}{{\mathbf c}}
\newcommand{\cu}{{\mathbf u}}

\newcommand{\encod}{g}
\newcommand{\head}{f}
\newcommand{\thead}{\head}  
\definecolor{mostharm}{HTML}{FF0000}
\definecolor{medianharm}{HTML}{000000}
\definecolor{leastharm}{HTML}{7B2CBF}

\renewcommand{\E}{\mathbb{E}}
\newcommand{\MI}{\mathrm{I}}

\newcommand{\CFR}[1]{{\color{black} #1}}
\newcommand{\mypara}[1]{\textbf{#1}\hspace{0.5em}}

\title{Harmful Partial Interventions and Benign Leakage in Concept-based Models}

\author{Gregory Verghese\\
King’s College London, UK\\
PharosAI, UK \And
Christopher J. Soelistyo\\
The Francis Crick Institute, UK
\And Mateo Espinosa Zarlenga\\
University of Oxford, UK
\And Alessandro Barp\\
University College London, UK\\
The Alan Turing Institute, UK
\AND Christopher R.S. Banerji\\
King’s College London, UK\\
The Alan Turing Institute, UK\\
\texttt{christopher.banerji@kcl.ac.uk}}
\iclrfinalcopy 
\begin{document}

\maketitle

\begin{abstract}
Artificial intelligence systems deployed in high-risk settings increasingly require models that support meaningful human oversight. Concept-based models (CMs) offer a promising framework by enabling users to inspect and intervene on intermediate, human-understandable concepts. However, concept interventions can counterintuitively worsen downstream predictions, a failure mode we formalise as \textit{intervention harm}. We distinguish harm after all concepts have been intervened on from harm under partial interventions, the more likely practical use case. Modern CMs often achieve high accuracy by encoding, or \textit{leaking}, additional task-relevant information in their intermediate concept representations. Recent work argues that such leakage can be benign with respect to intervention harm when concept representations satisfy two conditions: sufficiency and localisation. We prove that these criteria are not sufficient to guarantee that concept interventions improve downstream performance. Across controlled and real-world settings, we find that concept supervision incompleteness is associated with increased intervention harm. Importantly, optimising the model when all concepts are intervened on, a previously proposed leakage mitigation objective, reduces endpoint harm but does not reliably prevent harmful partial corrections. Hence, we introduce a complementary monotonic intervention objective that directly penalises increases in task loss along partial intervention trajectories. Combining endpoint and monotonic intervention training substantially reduces the prevalence and magnitude of partial intervention harm while preserving concept and task predictive performance.
 \end{abstract}
 
\section{Introduction}
\label{sec:intro}

Artificial intelligence (AI) deployed in high-risk settings is increasingly expected to support interpretability, transparency and meaningful human oversight. Concept-based models (CMs) offer a promising framework by structuring predictions around intermediate, human-interpretable concepts \citep{koh2020concept}. Beyond making model reasoning more inspectable, concepts provide an interface through which a human overseer can intervene on the model's reasoning, thereby enabling users to correct erroneous intermediate concept predictions
\citep{shin2023closer, espinosa2023learning}. For such models to support meaningful oversight, however, intervening on a model's intermediate concept predictions should have predictable and beneficial effects on downstream task performance.\looseness=-1

A fundamental limitation of earlier CMs (e.g., Concept Bottleneck Models~\citep{koh2020concept}) is that they require a set of supervised concept labels that completely explain the downstream task. In practice, however, concept sets tend to be \emph{incomplete}, capturing only part of the information needed to predict the task and causing CMs with strict concept bottlenecks to underperform~\citep{mahinpei2021promises}. Therefore, several works have introduced architectural pathways that explicitly \textit{leak} task-relevant information that is missing from the concept set into the task predictor~\citep{espinosa2022concept, yuksekgonul2022post, margeloiu2021concept, havasi2022addressing, parisini2025leakage}. 
Although such leakage can improve model accuracy, it \textit{may} undermine the effectiveness of human intervention. When a concept representation contains leaked information, correcting the predicted concept to its ground-truth value may move the model's latent representation in a direction that does not correspond to the human's intended correction~\citep{almudevar2026there}. If this intervention worsens the downstream prediction, we refer to the resulting failure mode as \emph{intervention harm}.\looseness=-1

Recently, \emph{benign leakage} was introduced to describe leakage that allows models to retain high concept and task accuracies without displaying intervention harm~\citep{zarlenga_defense_2026}. Alongside this notion, \emph{sufficiency} and \emph{localisation} were proposed as criteria for identifying benign leakage. Informally, sufficiency requires the learned concept representations to retain task-relevant information in the input features, while localisation requires assignment of this information to the components of the concept representation that vary under intervention. These conditions motivate approaches for encouraging benign leakage during training: minimising the intervention loss $\mathcal{L}_{\mathrm{int}}$, defined as the expected task loss after all concept representations are intervened on, is, under certain assumptions, equivalent to maximising sufficiency and has been shown to improve task performance under intervention~\citep{zarlenga_defense_2026}. 
However, whether these criteria are sufficient
to prevent intervention harm remains to be rigorously established.

Crucially, interventions in real-world scenarios are rarely exhaustive, as CMs may contain hundreds of concepts, and users generally cannot be expected to verify every concept. Hence, interventions are typically \emph{partial}: an overseer may inspect and correct only a small number of concepts considered relevant to the task at hand.
However, reliable partial intervention requires a stronger notion of intervenability than that driven by $\mathcal{L}_{\mathrm{int}}$. In practice we want CMs to avoid performance degradation when \emph{any} subset of erroneous concepts is corrected, even when other concepts remain uncorrected. This exposes a distinction between \emph{endpoint intervention harm}, which $\mathcal{L}_{\mathrm{int}}$ minimises, and \emph{partial intervention harm}, which remains understudied, particularly for leaky CMs where benign leakage is encouraged via $\mathcal{L}_{\mathrm{int}}$. In incomplete settings, the former is actually a special case of the latter, since the supervised concept set does not extend to all task-relevant concepts. Therefore, we believe it is critical to better understand the prevalence and consequences of partial intervention harm.\looseness=-1

With this in mind, this work develops a theoretical and empirical framework for studying endpoint and, more importantly, partial intervention harm. Specifically, our contributions are: (1) we show that sufficiency and localisation do \emph{not} guarantee that correcting a concept improves task performance; (2) we identify concept-set incompleteness as an important determinant of intervention harm;
(3) we distinguish \emph{endpoint} from \emph{partial} intervention harm and show that $\mathcal{L}_{\mathrm{int}}$, while effective at reducing endpoint harm, does not reliably prevent partial harm, especially in incomplete settings; and (4) we introduce $\mathcal{L}_{\mathrm{mono}}$, a complementary objective that penalises increases in task loss along partial intervention paths. Combining $\mathcal{L}_{\mathrm{mono}}$ and $\mathcal{L}_{\mathrm{int}}$ substantially reduces the prevalence and magnitude of partial intervention harm while preserving concept and task predictive performance. Together, these results provide a practical approach to encourage benign leakage under incompleteness while preserving high task and concept accuracy and avoiding intervention harm, even under partial interventions.\looseness=-1

\section{Related Work}

\mypara{Concept-based Models.}
Although diverse in architecture and purpose~\citep{knab2026s}, most Concept-based Models (CMs) are best framed as Concept Bottleneck Models (CBMs)~\citep{koh2020concept}. At a high level, a CBM $(\encod, \head)$ is a compositional deep neural network (DNN) formed by a \textit{concept encoder} $\encod: \mathcal{X} \rightarrow \mathcal{C}$ and a \textit{task predictor} $\head: \mathcal{C} \rightarrow \mathcal{Y}$. Intuitively, the concept encoder $\encod(\mathbf{x})$ maps input features $\mathbf{x} \in \mathcal{X}$ to a set of \textit{concept representations} $\widehat{\ck} \in \mathbb{R}^{k \times d}$ aligned with human-interpretable units of information $\ck \in \{0, 1\}^k$ (e.g., colours, textures). The task predictor then uses these representations to predict a downstream task $\widehat{y} = \head(\widehat{\ck})$. Concept representations can be scalars (i.e., $d = 1$)~\citep{chen_concept_2020, havasi2022addressing,  yuksekgonul2022post, vandenhirtz2024stochastic}, or richer high-dimensional vectors (i.e., $d > 1$)~\citep{espinosa2022concept, kim_probabilistic_2023, xu2024energy}.
In practice, all CMs require a concept-annotated training set of triplets $(\mathbf{x}, \ck, y)$ to align their intermediate representations with concepts $\ck$ useful for a downstream task $y$. These labels may come from human-annotated datasets~\citep{ramaswamy2023overlooked, skirzynski2026measuring}, be extracted from a pre-trained model's representations~\citep{alvarez-melis_towards_2018, yuksekgonul2022post, de2026mechanistic_cbm}, or be discovered using foundation models~\citep{oikarinen_label-free_2023, yang2023labo, rao2024discover}.\looseness=-1


\mypara{Concept Interventions.}
A key property of CBMs is that they support \emph{concept interventions}: at test time, an expert can correct mispredicted concept scores, or inject concept-based counterfactuals in a CM's reasoning, by directly modifying a concept representation $\widehat{\ck}_i$ so that it is aligned with its corresponding ground-truth concept $c_i$. Generally, concept interventions are expected to increase a CM's downstream task accuracy~\citep{koh2020concept}. However, recent works have shown that not all concept interventions lead to equivalent  improvements~\citep{chauhan_interactive_2023, espinosa2023learning}, and some may even decrease a model's performance in ill-conditioned CBMs~\citep{koh2020concept, zarlenga2023towards}. As a result, follow-up research has explored how to improve the impact of interventions by (1) devising mechanisms to propagate interventions~\citep{vandenhirtz2024stochastic, singhi2024improving}, (2) learning to reapply previous interventions~\citep{steinmann2023learning}, or (3) introducing the ability for the CM to explicitly defer concepts to experts~\citep{pugnana2026deferring}.\looseness=-1

\mypara{Information Leakage.}
Concept representations can encode more information about other concepts or the downstream task than their corresponding ground-truth concepts do, a phenomenon called \emph{information leakage}~\citep{margeloiu2021concept, mahinpei2021promises}. Intuitively, leakage can be seen as a form of shortcut learning, where the way a CM uses a concept misaligns with how a human expert would use it~\citep{geirhos2020shortcut, marconato2023not, enouen_debugging_2026}. As such, leakage is more likely to appear when CMs are trained on \textit{incomplete} concept sets that do not fully describe the downstream task. Properly understanding leakage, when it happens, and its consequences has led to a rich literature on the topic. Works in this area vary from proposing ways of measuring leakage~\citep{zarlenga2023towards, makonnen_measuring_2025, parisini2025leakage}, to exploring mechanisms for detecting~\citep{marconato2022glancenets, zarlenga2025avoiding}, controlling~\citep{ragkousis2024tree, parisini2026concept}, and mitigating~\citep{havasi2022addressing, almudevar2026there} leakage. More relevant to understanding partial intervention harm, \citet{zarlenga_defense_2026} recently argued that leakage can be \textit{benign}, in the sense that it does not affect a CM's performance and intervenability, when concept representations are \textit{sufficient} and \textit{localised}. There, the authors argue that these properties can be encouraged by minimising the task loss when all concepts are intervened on, a loss they call $\mathcal{L}_\text{int}$.\looseness=-1


\section{Understanding Intervention Harm}

\CFR{In this paper, we show that, although useful, sufficiency and localisation criteria are \textit{not} sufficient to guarantee no intervention harm when leakage is allowed in a CM's representations. To show this, we first describe our problem setup and then formalise sufficiency, localisation, and intervention harm.\looseness=-1}

\subsection{Problem Setup}
\label{sec:known-unknown-concepts}

Let $\mathbf{x} \in \mathcal{X} \subseteq \mathbb{R}^n$ denote an $n$-dimensional feature representation (e.g., image pixels) of a realisation of random variable (r.v.) $X$. Throughout this paper, we study supervised learning tasks where each sample $\mathbf{x}$ is associated with a task label $y \in \mathcal{Y} \subseteq \mathbb{R}$ (e.g., bird species), which itself is a realisation of a r.v.\ $Y$. For each sample $\mathbf{x}$, let $\ck \in \mathcal{C} \subseteq \{0, 1\}^k$ be a vector of $k$ known binary concept labels relevant to $Y$ (e.g., ``has red wing'', ``has hooked beak''). We assume that each binary concept label $\ck_i$ in this vector is a realisation of a r.v.\ $C_i$, and use $C$ to represent the r.v.\ corresponding to the entire known concept set $\ck$. Here, we say a concept set $\mathcal{C}$ is \textit{incomplete} when $\ck$ cannot fully explain the task $y$. From an information-theoretic perspective, incompleteness therefore implies $H(Y \mid C) > 0$.\looseness=-1

Throughout this paper, we work under the realistic assumption that the set of training concept labels for a CM is incomplete. In this case, we use $\mathbf{u} \in \mathcal{U} \in \{0, 1\}^{k^\prime}$ to describe a vector containing all concepts not in $\mathcal{C}$ that are relevant to the task (which we think of as realisations of a r.v.\ $U$ representing all \textit{unknown} task-relevant concepts).  Hence the task outcome is a function of $\mathcal C \times \mathcal U$. 
As such, we assume there is an optimal CM $(\head^* \circ \encod^*): \mathcal{X} \rightarrow \mathcal{Y}$ that first predicts the known and unknown concepts  $\encod^*(\mathbf{x}) = (\ck ,\cu ) \in \Ccal \times \Ucal$ from the input features $\mathbf{x} \in \mathcal X$, and then maps $(\ck ,\cu )$ to a downstream task prediction $\widehat{y}$ using a second function $\head^*(\ck ,\cu ) = \widehat{y}$.
We emphasise that, in contrast to this optimal predictor, a real-world CM trained on this task can observe only triples $(\mathbf x,\ck ,y)$ during training, while the corresponding unknown-concept state $\cu $ remains unobserved.\looseness=-1

\mypara{Residual and Hybrid CBMs.}
\label{sec:concept-residual-representation}
To overcome the information bottleneck that incompleteness imposes, several works propose introducing residual information pathways that the CM can use to capture information in $Y$ that is not present in the training concepts $C$~\citep{havasi2022addressing, espinosa2022concept, yuksekgonul2022post, sawada_concept_2022}. A common implementation of this strategy is a Hybrid CBM~\citep{mahinpei2021promises}, where the output of the concept encoder $\encod$ is extended so that it produces both a set of aligned representations $\widehat{\ck}$ and a set of unaligned, free \textit{residual} representations $\mathbf{z}$. Given the expressivity of this approach, we study leaky CMs in concept-incomplete settings using Hybrid CBMs as our main object of study. Therefore, from here on, we extend the CBM-based formulation of CMs so that the concept encoder $ \encod (\mathbf{x}) $ explicitly outputs both a set of soft concept scalar representations $\widehat{\ck} \in [0, 1]^k$ and a residual pathway representation $\mathbf{z} \in \mathbb{R}^{\mu}$. Similarly, we assume the label predictor's input now includes this residual path (i.e., $\head$ now maps $\mathcal{C} \times \mathbb{R}^{\mu}$ to $\mathcal{Y}$). This enables us to work with expressive leaky CMs that can still operate in incomplete settings, while simplifying concept representations to scalars and explicitly encoding the residual path.\looseness=-1

\subsection{Benign Leakage, Sufficiency, and Localisation}
\label{sec:benign-leakage}
Introducing a residual pathway like that seen in Hybrid CBMs risks losing the model's intervenability~\citep{almudevar2026there}, defined as its ability to show monotone increases in task accuracy as one increases the number of concept representations aligned to their ground truth values~\citep{vandenhirtz2024stochastic}. Nevertheless, recent work by \citet{zarlenga_defense_2026} argues that leakage does not necessarily need to affect a model's intervenability, and can therefore be \textit{benign}, so long as two information-theoretic conditions are satisfied:\looseness=-1

\mypara{Sufficiency.}
The first claimed condition for benign leakage, termed \textit{sufficiency}, can be described as:
\begin{equation}
    \MI(Y ; Z \mid C )
    \approx
    \MI(Y ; X\mid C ),
    \label{eq:zarlenga-sufficiency}
\end{equation}
where $I(\cdot)$ is the mutual information, and $Z$ is the r.v. representing the residual representations $\mathbf{z}$. Intuitively, this condition requires that all information about the downstream task $Y$ that is present in the features $X$ but not in the concepts $C$ be encoded in the residual representations $\mathbf{z}$.
To capture how closely sufficiency is satisfied, we define the \textit{sufficiency discrepancy} as
\begin{equation}
    \Delta_{\mathrm S} = \big| \MI(Y ; X \mid C ) - \MI(Y ; Z \mid C ) \big|.
\end{equation}
Sufficiency holds exactly when $\Delta_{\mathrm S}=0$, and approximately when $\Delta_{\mathrm S}$ is small.\looseness=-1

\mypara{Localisation.}
The second condition deemed important for leakage to be benign is \textit{localisaton}.
For Hybrid CBMs, our main object of study, each ground-truth concept $c_i$ is represented by a single scalar dimension $\widehat{c}_i$, and the model generates a single, explicit residual representation $\mathbf{z}$. Hence, one can think of each concept representation $\widehat{c}_i$ as associated with the same shared block of residual variables $\mathbf{z}$. In this context, localisation can be formally expressed as:\looseness=-1
\begin{equation}
    \MI\!\left(
        Y; C_i
        \mid C_{-i}
    \right)
    \approx
    \MI\!\left(
        Y ; \widehat C_i
        \mid
        Z, \widehat{C}_{-i}
    \right),
    \label{eq:zarlenga-localisation}
\end{equation}
where $\widehat{C}$ is the r.v. associated with the concept representations $\widehat{\ck}$, and the subscript \(-i\) denotes all components in a variable except \(i\). Intuitively, this condition requires all information about a known concept $c_i$ to be localised into its corresponding representation $\widehat{c}_i$, rather than spread across the shared residual variables $\mathbf{z}$.
Similarly to sufficiency, we define the localisation discrepancy for concept \(i\) as \begin{equation}
    \Delta_{\mathrm L,i} = \big|\MI(Y; C_i \mid C_{-i}) - \MI(Y ; \widehat C_i \mid Z, \widehat{C}_{-i})\big|.
\end{equation}
Localisation holds exactly when \(\Delta_{\mathrm L,i}=0\), and approximately when \(\Delta_{\mathrm L,i}\) is small.\looseness=-1

\subsection{Intervention Harm}
To better study intervenability under benign leakage, it is helpful to first define intervention harm:\looseness=-1
\begin{definition}[Intervention harm]
\label{def:intervention-harm}

Consider a (potentially residual) CM and let $(\widehat{\mathbf c},\mathbf z)$ denote the model-predicted representation before intervention. For any subset of known concepts $S\subseteq \mathcal{C}$, define the oracle-intervened concept representation $\mathbf{\widehat c}^{(S)}$ by the vector of concept representations where
\[
    \widehat c^{(S)}_i
    :=
    \begin{cases}
        c_i, & i\in S,\\
        \widehat c_i, & i\notin S,
    \end{cases}
\]
Here we assume that, if the CM is a residual model, the residual representation $\mathbf z$ is held fixed during intervention.
Notice that, by construction, $\mathbf{\widehat{c}}^{(\varnothing)}=\mathbf{\widehat{c}}$. For a measurable task loss $L:\mathcal Y\times \mathcal Y\to\mathbb R$, define the \textit{intervention task risk} $\Risk_{S}$ as 
\begin{equation}
    \Risk_{S}
    :=
    \E_{(\mathbf{x}, \ck, y) \sim P_{\mathcal{D}}}\!\left[
        L\!\left(
            y,
            \thead(\widehat{\mathbf c}^{(S)},\mathbf z)
        \right)
    \right],
    \label{eq:intervention-pattern-risk}
\end{equation}
where $P_{\mathcal{D}}$ is the data distribution. 
Assuming $\Risk_{\varnothing}$ is finite, the expected intervention harm associated with correcting the concepts in $S$ is therefore defined as
$
    \eta(S)
    :=\Risk_{S}-\Risk_{\varnothing}$.
    \label{eq:general-intervention-harm}
 We say that intervention on $S$ is \emph{harmful} if $\eta(S)>0$ and
\emph{safe} if $\eta(S)\leq 0$. When $S\subsetneq \mathcal{C}$, $\eta(S)$ describes
\emph{partial intervention harm}; when $S=\mathcal{C}$, it describes
full, or \emph{endpoint}, harm. By construction, $\eta(\varnothing)=0$.

\end{definition}

\subsection{Sufficiency and Localisation Do Not Prevent Intervention Harm} \label{sec:harm-under-benign-leakage} 
We now consider whether sufficiency and localisation are indeed \textit{sufficient} to prevent intervention harm and thus prescribe benign leakage. We find that for a concept classifier with non-zero predictive error (the setting where concept interventions are warranted), intervention harm can be arbitrarily large even when sufficiency and localisation conditions hold arbitrarily well:\looseness=-1

\begin{theorem}[Interventions can be arbitrarily harmful under approximate benign leakage]
\label{thm:noisy-xor-correction-harm}

For every \(p\in(0,1/2)\) and every \(n>0\), there exists a finite data-generating process, concept representations $\widehat \ck \sim \widehat C$, residuals $\mathbf{z} \sim Z$ such that: \textbf{(i)} $\widehat C_i$ is misaligned with ground-truth concept $C_i$ with probability \(p\); \textbf{(ii)} localisation holds exactly
but sufficiency discrepancy is $\Delta_{\mathrm S}=\mathfrak{h}(p)$, where \(
        \mathfrak{h}(p)
        :=
        -p\log p-(1-p)\log(1-p)
    \) 
is the binary entropy; \textbf{(iii)} the  prediction of the optimal task classifier $\head$ is incorrect with probability \(p\);  and, \textbf{(iv)} replacing \(\widehat C\) by its ground-truth value \(C\), while holding \(Z\) fixed, strictly increases expected squared loss:
    \begin{equation}
        \E\!\left[
            \bigl(
                Y-\thead(C,Z)
            \bigr)^2
        \right]
        -
        \E\!\left[
            \bigl(
                Y-\thead(\widehat C,Z)
            \bigr)^2
        \right]
        \geq
        np
        >
        0.
        \label{eq:noisy-xor-positive-harm}
    \end{equation}



Equivalently, this implies that, for every \(\varepsilon>0\) and every \(M>0\), there exists an example that satisfies 
\begin{equation}
    \Delta_{\mathrm S}<\varepsilon,
    \qquad
    \Delta_{\mathrm L,1}=0,
    \qquad
    \eta(\mathcal{C})>M.
    \label{eq:noisy-xor-arbitrarily-benign}
\end{equation}
\end{theorem}

See Appendix~\ref{proof_thm_1} for a proof. Theorem \ref{thm:noisy-xor-correction-harm} shows that localisation and arbitrarily approximate sufficiency are not sufficient conditions to guarantee benign leakage, and intervention harm can be arbitrarily large in their presence. The key proof construction shows that correcting $\widehat C$ while holding $Z$ fixed can change how the task head uses the joint representation, relocating rather than removing prediction error. Crucially, this mechanism extends to partial intervention, because it arises from a concept-correction transition rather than a property of endpoint concept  intervention, motivating the study of \emph{partial} intervention trajectories.\looseness=-1

\section{Minimising Partial Intervention Harm}
\label{sec:intervention-loss}

Recently, \citet{zarlenga_defense_2026} showed that, under certain idealised conditions, sufficiency can be maximised by minimising the intervention loss \(\mathcal{L}_{\mathrm{int}}=\Risk_{\mathcal C}\), where \(\Risk_{\mathcal C}\) is the task risk after ground-truth intervention on all supervised concepts. Minimising $\mathcal{L}_{\mathrm{int}}$ has also been shown to improve model performance following interventions. However, our theoretical results show that sufficiency and localisation are not sufficient to prevent partial intervention harm.
Moreover $\mathcal{L}_{\mathrm{int}}$, places no direct constraint on intermediate states in which only a subset of concepts has been corrected. Hence, while it may aid in preventing intervention harm when all concepts are intervened, which we refer to as \textit{endpoint} intervention harm, it cannot be expected to prevent \textit{partial} intervention harm.\looseness=-1

As we discuss in greater detail in Section~\ref{sec:experiments}, we empirically observe that $\mathcal{L}_\text{int}$ alone is not sufficient to prevent partial intervention harm in incomplete settings. Since partial interventions are the expected use case for CMs in practice, it is important to understand what training incentives may eliminate partial intervention harm. Therefore, we argue that a stronger requirement to limit partial intervention harm is \emph{monotonicity}, where correcting an additional concept should not increase task risk. This implies that, for every \(S\subseteq\mathcal C\) and \(i\in\mathcal C\setminus S\), we want $\Risk_{S\cup\{i\}}-\Risk_S \leq 0.$ If this condition holds for every such pair, risk is non-increasing along every sequence of concept corrections; hence, no partial or full interventions can be harmful. \looseness=-1

\mypara{Monotonic Intervention Loss.}
Inspired by this notion, we introduce the \emph{monotonic intervention loss} $\mathcal L_{\mathrm{mono}}$. To understand this loss, consider an \emph{intervention edge}
$
    S\rightarrow S\cup\{i\}, S\subsetneq\mathcal C, i\in\mathcal C\setminus S,
$ where $S$ is a set of already-intervened concepts and $i$ is an additional concept to be corrected. For each example, define
$
    \ell_S(\widehat{\mathbf{c}}, \mathbf{z}, \ck, y)
    :=
    L\!(
        y,
       \thead\!(
            \widehat{\mathbf c}^{(S)}_{\mathrm{sg}},
            \mathbf z
        )
    ),
$
where \(\operatorname{sg}(\cdot)\) denotes the stop-gradient operator and \(\widehat{\mathbf c}^{(S)}_{\mathrm{sg}}\) is constructed from \(\operatorname{sg}(\widehat{\mathbf c})\).
Our loss is then given by:
\begin{equation}
    \mathcal L_{\mathrm{mono}}
    =
    \E_{(S,i)\sim q}
    \E_{(\mathbf{x}, \ck, y) \sim P_\mathcal{D}}
    \left[
        \ell_{S\cup\{i\}}\left(\encod(\mathbf{x}), \ck, y\right) - \ell_S\left(\encod(\mathbf{x}), \ck, y\right)
    \right]_+
    \label{eq:mono-loss}
\end{equation}
where $[u]_+:=\max(u,0)$, \(q(S,i)\) specifies a distribution over intervention edges and we use $\ell_S\left(\encod(\mathbf{x}), \ck, y\right)$ to represent $\ell_S\left(\widehat{\mathbf{c}}, \mathbf{z},\ck,  y\right)$ when $\encod(\mathbf{x}) = (\widehat{\ck}, \mathbf{z})$. Thus, \(\mathcal L_{\mathrm{mono}}\) is zero whenever correcting one additional concept maintains or improves task performance and positive otherwise.\looseness=-1

Notice that we stop gradients through \(\widehat{\mathbf c}\) when optimising \(\mathcal L_{\mathrm{mono}}\). This prevents the model from reducing the monotonicity penalty by changing the supervised concept predictions themselves, which could trade concept fidelity for intervention compatibility. Gradients are retained through the residual representation \(\mathbf z\) and task head, allowing the downstream representation and prediction function to adapt so that ground-truth concept corrections are interpreted consistently.\looseness=-1

\mypara{Generating Intervention Trajectories.}
Computing $L_{\mathrm{mono}}$ becomes intractable with a large concept set, as the number of intervention edges increases exponentially with $k$. Hence, during training we estimate Eq.~\ref{eq:mono-loss} by Monte Carlo sampling. Specifically, for each minibatch we sample \(m_{\mathrm{mono}}\) intervention edges independently from \(q(S,i)\), where \(m_{\mathrm{mono}}\) is a training hyperparameter, and average their monotonicity penalties. In our experiments, we sample intervention depth uniformly: first
$
    |S|\sim
    \operatorname{Uniform}\{0,\ldots,|\mathcal C|-1\},
$ then a random concept ordering, taking the first \(|S|\) concepts to form \(S\) and the next concept as \(i\). Consequently, all intervention depths, from the first concept correction to the final correction, are represented during training.


\mypara{Integrating $\mathcal{L}_\text{mono}$.}
Our loss can be integrated into the traditional CBM training objective~\citep{koh2020concept}, extended to include $\mathcal{L}_\text{int}$, as follows:
\begin{equation}
    \mathcal L
    =
    \mathcal L_{\mathrm{task}}
    +
    \lambda_c\mathcal L_{\mathrm{concept}}
    +
    \lambda_{\mathrm{int}}\mathcal L_{\mathrm{int}}
    +
    \lambda_{\mathrm{mono}}\mathcal L_{\mathrm{mono}}.
    \label{eq:combined-mono-objective}
\end{equation}
Here, \(\lambda_c,\lambda_{\mathrm{int}},\lambda_{\mathrm{mono}}\geq0\) control the contribution of concept prediction, fully intervened task performance, and monotonic partial-intervention performance, respectively. Intuitively, \(\mathcal L_{\mathrm{int}}\) encourages good performance at the fully corrected endpoint, whereas \(\mathcal L_{\mathrm{mono}}\) directly penalises harmful transitions between intervention states. As we show next, these two losses complement each other.

\section{Experiments}
\label{sec:experiments}

In this section, we empirically show that partial intervention harm can occur even when CMs are trained with $\mathcal{L}_\text{int}$. Then, we show that introducing the $\mathcal{L}_\text{mono}$ loss defined in Section~\ref{sec:intervention-loss} can significantly mitigate these effects and is a complementary objective mitigating different aspects of harm w.r.t.\ to $\mathcal{L}_\text{int}$. To do this, we explore Hybrid CBMs with varying residual capacity 
on both synthetic datasets (Section~\ref{sec:synthetic-experiments}), where we fully control the data generation process, and real-world datasets (Section~\ref{sec:cub-experiments}). Across all experiments, we simulate incompleteness by controlling the size of the set of known concepts $\mathcal{C}$ used when training a Hybrid CBM. Appendix~\ref {app:experimental_details} provides details on model configurations, hyperparameters, and datasets.\looseness=-1

\subsection{Understanding Intervention Harm, Leakage, and Incompleteness}
\label{sec:synthetic-experiments}

\mypara{Setup.}
We first study intervention harm using a controlled linear--Gaussian model adapted from \citet{parisini2026concept} (Appendix~\ref{app:synthetic-methods}). In this setup, models are provided with $k=5$ supervised concepts while the task depends on some additional \(k^\prime\in\{2,5,10,15\}\) unobserved concepts. We vary the relative unobserved concept task contribution via a parameter \(\omega\in\{0.10,0.25,0.50,0.75,0.90\}\) (higher $\omega$ means higher task-relevance of unobserved concepts), and generate datasets using three concept covariance structures \CFR{(all concepts independent, block independence of observed and unobserved concepts, and random concept covariance)}. For each configuration, we train Hybrid CBMs with different \(\lambda_{\mathrm{int}}\) and \(\lambda_{\mathrm{mono}}\), and with and w/o stop-gradient through \(\widehat{\mathbf c}\), on five seeds. 
When studying harm, we exhaustively evaluate all \(2^5\) intervention states and \(5!\) correction orderings.

\begin{figure}[t!]
    \centering
    \includegraphics[width=\linewidth]
        {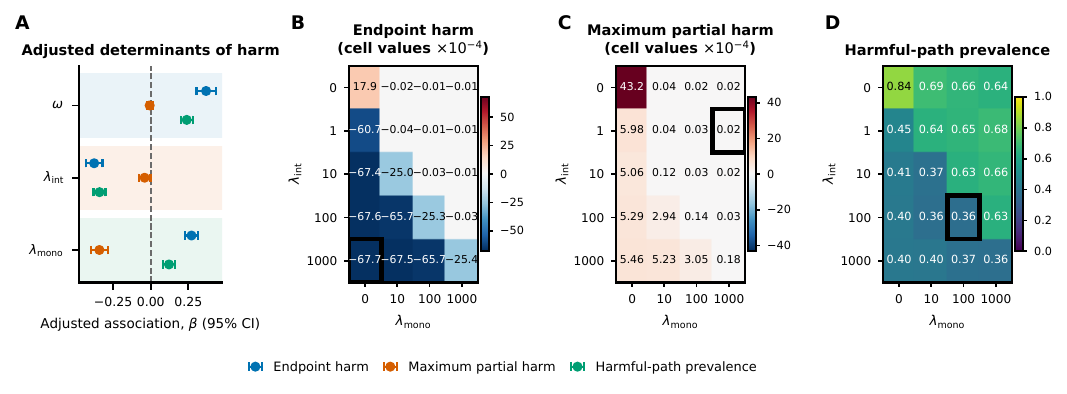}
    \caption{Endpoint ($\mathcal{L}_\text{int}$) and monotonic ($\mathcal{L}_\text{mono}$) objectives control complementary components of intervention harm.
    \textbf{(A)} Adjusted standardised associations of $\omega$, $\lambda_{\mathrm{int}}$ and $\lambda_{\mathrm{mono}}$ with endpoint harm, maximum partial harm, and harmful-path prevalence. Points show regression coefficients with 95\% confidence intervals.
    \textbf{(B)} Mean endpoint harm, \textbf{(C)} maximum partial harm, and \textbf{(D)} harmful-path prevalence across \(\lambda_{\mathrm{int}}\) and \(\lambda_{\mathrm{mono}}\).
    In these plots, we highlight the cell with the lowest mean.}
    \label{fig:synthetic-results}
    \vspace{-0.5em}
\end{figure}

\mypara{As unobserved concepts become more important, intervention harm increases.}
Our results in Fig.~\ref{fig:synthetic-results}A, where we show the association of each main model component to observed intervention harm, suggest that greater task dependence on unobserved concepts is associated with increased intervention harm. Specifically, we see that, in adjusted analyses, increasing \(\omega\) was associated with greater endpoint harm (\textcolor{NavyBlue}{blue}, \(\beta=0.368\), 95\% CI \(0.304\)--\(0.431\), \(P<10^{-28}\)) and a higher fraction of harmful partial-correction paths (\textcolor{ForestGreen}{green}, \(\beta=0.239\), 95\% CI \(0.200\)--\(0.278\), \(P<10^{-32}\)), but not with maximum partial harm (\textcolor{orange}{orange}, \(\beta=-0.007\), 95\% CI \(-0.027\)--\(0.013\), \(P=0.47\)). This suggests that incompleteness plays a significant role in the presence of both partial and endpoint intervention harm.\looseness=-1

\mypara{$\mathcal{L}_\text{int}$ and $\mathcal{L}_\text{mono}$ are complementary losses.}
Interestingly, our results in Figs.~\ref{fig:synthetic-results}B--D suggest that \(\mathcal L_{\mathrm{int}}\) and \(\mathcal L_{\mathrm{mono}}\) control complementary components of harm. 
Without either objective, both endpoint and partial intervention harm are positive (Fig.~\ref{fig:synthetic-results}B--D) and 84\% of intervention paths are harmful. Adding $\mathcal L_{\mathrm{int}}$ makes endpoint harm negative, but positive partial harm persists. Adding $\mathcal L_{\mathrm{mono}}$ to $\mathcal L_{\mathrm{int}}$ reduces maximum partial harm dramatically by up to 99.7\%, however $\mathcal L_{\mathrm{mono}}$ alone cannot reduce endpoint harm as effectively as $\mathcal L_{\mathrm{int}}$. Consistent with this complementarity, increasing $\lambda_{\mathrm{int}}$ primarily reduces endpoint harm ($\beta=-0.376$, $P<10^{-39}$), whereas increasing $\lambda_{\mathrm{mono}}$ primarily reduces maximum partial harm
($\beta=-0.341$, $P<10^{-33}$). Hence, balancing both objectives provides a strong reduction in general intervention harm.\looseness=-1

\mypara{Intervention-aware losses do not affect predictive performance.}
Finally, we see that the improvements in intervention harm do not need to come at the cost of either concept accuracy or task accuracy. As discussed in more detail in Appendix \ref{app:synthetic-robustness}, with moderate values of $\lambda_\text{int}$ and $\lambda_\text{mono}$, the task accuracy changes by at most 0.043\%. Moreover, the concept mean squared error (MSE) remains unaffected when $\mathcal{L}_\text{mono}$ uses its stop-gradient mechanism
(Appendix \ref{app:synthetic-robustness}).\looseness=-1

Together, these experiments show that concept incompleteness produces distinct endpoint- and partial-intervention failure modes, which can be mitigated by jointly optimising \(\mathcal{L}_{\mathrm{int}}\) and \(\mathcal{L}_{\mathrm{mono}}\).\looseness=-1

\subsection{Intervention Harm in Real-World Tasks}
\label{sec:cub-experiments}

\begin{figure}[t!]
    \centering
    \includegraphics[width=\linewidth]{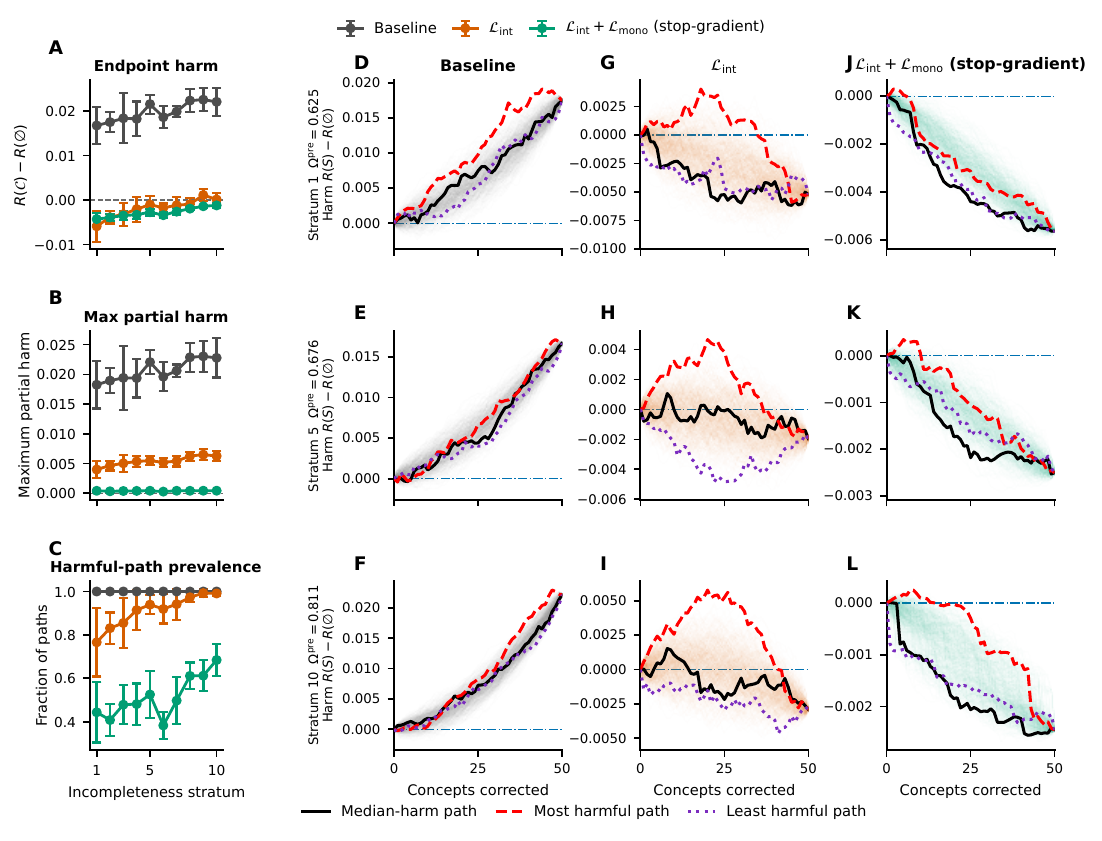}
    \caption{$\mathcal L_{\mathrm{mono}}$ suppresses harmful partial intervention trajectories on \texttt{CUB}. \textbf{(A)} Mean endpoint harm, \textbf{(B)} maximum partial harm, and \textbf{(C)} harmful-path prevalence across ten concept incompleteness strata (higher means higher incompleteness $\Omega_{\mathcal C}^{\mathrm{pre}}$). We show
    Hybrid CBMs trained (1) w/o intervention-aware losses (\textit{\textcolor{darkgray}{baseline}}) , (2) with \textcolor{orange}{$\mathcal{L}_{\mathrm{int}}$}, and (3) with \textcolor{ForestGreen}{$\mathcal{L}_{\mathrm{int}} + \mathcal{L}_{\mathrm{mono}}$}. Points show means, and bars show 95\% confidence intervals.
    \textbf{(D--L)} We show partial-intervention trajectories
    from \textbf{(D, G, J)} low-, \textbf{(E, H, K)} intermediate-, and \textbf{(F, I, L)} high-incompleteness strata. Columns show intervention paths for different modalities: \textbf{(D--F)} baseline, \textbf{(G--I)} $\mathcal{L}_\text{int}$ only, and \textbf{(J--L)} \(\mathcal L_{\mathrm{int}}+\mathcal L_{\mathrm{mono}}\). 
    Faint lines show all 512 sampled intervention paths, with the \textcolor{medianharm}{median}, \textcolor{mostharm}{most harmful}, and \textcolor{leastharm}{least harmful} paths highlighted.
    Harm is measured relative to the observational state, \(R(S)-R(\varnothing)\) (i.e., higher is worse).
    }
    \label{fig:cub-trajectories}
    \vspace{-0.5em}
\end{figure}

\mypara{Setup.}
We next test whether our findings above generalise on \texttt{CUB}~\citep{wah2011caltech}, a real-world vision dataset containing \(11{,}788\) images from 200 bird species and 312 binary image-level concepts. To study the effect of concept incompleteness in this task, we construct 80 sets containing 50 concepts and quantify the task-relevant information left unexplained by each set using a pre-model incompleteness score, \(\Omega_{\mathcal C}^{\mathrm{pre}}\) (Appendix ~\ref{app:cub-methods}). For each set, we train three Hybrid CBMs: one without intervention losses (\emph{baseline}), one trained with \(\mathcal L_{\mathrm{int}}\) (\emph{endpoint}), and one trained with \(\mathcal L_{\mathrm{int}}\) and \(\mathcal L_{\mathrm{mono}}\).
Because exhaustively evaluating all \(2^{50}\) intervention states is infeasible, we evaluate all models using the same 512 randomly sampled intervention orderings (see Appendix~\ref{app:cub-methods} for details).\looseness=-1

\mypara{Incompleteness strongly predicts the presence of harm in real-world data.}
As in our synthetic experiments, our results in Figs.~\ref{fig:cub-trajectories}A--C show a strong association between dataset incompleteness (captured by $\Omega_{\mathcal C}^{\mathrm{pre}}$) and both endpoint and partial-intervention harm. Specifically, we observe that increasing \(\Omega_{\mathcal C}^{\mathrm{pre}}\) is associated with greater endpoint harm (\(\beta=0.133\), 95\% CI \(0.084\)--\(0.181\), \(P<10^{-7}\)), maximum partial harm (\(\beta=0.078\), 95\% CI \(0.036\)--\(0.121\), \(P<10^{-3}\)), and harmful-path prevalence (\(\beta=0.194\), 95\% CI \(0.134\)--\(0.254\), \(P<10^{-9}\)). More importantly, we see that this harm is strongest when the Hybrid CBM is trained without any intervention-aware loss.
As real-world tasks are likely to be incomplete, these results suggest that, without proper conditioning (e.g., $\mathcal{L}_\text{int}$ and $\mathcal{L}_\text{mono}$ together), CMs are likely to exhibit both partial and endpoint intervention harm in practice.\looseness=-1

\mypara{$\mathcal{L}_\text{int}$ is not sufficient to avoid partial intervention harm.}
Our results in Fig.~\ref{fig:cub-trajectories}A show that training with \(\mathcal L_{\mathrm{int}}\) alone reduced mean endpoint harm from \(0.0198\) to \(-0.0019\). However, mean maximum partial harm remained \(0.0054\), and, crucially, \(91.2\%\) of sampled correction paths contained a harmful intermediate state (Fig.~\ref{fig:cub-trajectories}B--C). When looking at representative intervention trajectories across different incompleteness levels, shown in Figs.~\ref{fig:cub-trajectories}D--I, we see that \(\mathcal L_{\mathrm{int}}\) yields a harmless fully corrected state while still exhibiting partial intervention harm across several intermediate paths.\looseness=-1

In contrast, Figs.~\ref{fig:cub-trajectories}B--C show that adding \(\mathcal L_{\mathrm{mono}}\) (1) reduces mean maximum partial harm to \(0.0004\), a \(92.6\%\) reduction relative to training with \(\mathcal L_{\mathrm{int}}\) alone, and (2) reduces harmful-path prevalence from \(91.2\%\) to \(51.4\%\). In matched comparisons, the reductions are \(0.0050\) in maximum partial harm (95\% CI \(0.0047\)--\(0.0053\), \(P<10^{-49}\)) and 39.9 percentage points in harmful-path prevalence (95\% CI \(37.5\)--\(42.2\), \(P<10^{-47}\)). Corresponding trajectories show that \(\mathcal L_{\mathrm{mono}}\) substantially suppresses harmful intermediate excursions (Figs.~\ref{fig:cub-trajectories}G--L). All of these results suggest that, as in our synthetic task, $\mathcal{L}_\text{int}$ cannot guarantee no partial intervention harm (or benign leakage). Therefore, optimising across different intervention trajectories, as $\mathcal{L}_\text{mono}$ does, brings benefits beyond those brought by $\mathcal{L}_\text{int}$.\looseness=-1

\mypara{Model performance is preserved across training regimes.}
We further see that introducing intervention-aware losses has little effect on task accuracy: mean top-1 accuracy is \(77.94\%\), \(78.21\%\), and \(78.11\%\) for baseline, \(\mathcal L_{\mathrm{int}}\)-only, and \(\mathcal L_{\mathrm{int}}+\mathcal L_{\mathrm{mono}}\) models, respectively (see Appendix Fig.~\ref{fig:cub-performance}A). Although accuracy is slightly lower with combined training than with \(\mathcal L_{\mathrm{int}}\) alone (paired difference \(-0.10\) percentage points, 95\% CI \(-0.16\) to \(-0.04\), \(P=0.0024\)), it remained above baseline. Greater incompleteness is also associated with modestly reduced task accuracy (\(\beta=-0.186\), 95\% CI \(-0.347\)--\(-0.025\), \(P=0.024\); Appendix Fig.~\ref{fig:cub-performance}A). In terms of concept predictive performance, we see that adding \(\mathcal L_{\mathrm{mono}}\) reduces the concept MSE of endpoint training from \(0.1558\) to \(0.1461\) (paired reduction \(0.0097\), 95\% CI \(0.0087\)--\(0.0107\), \(P<10^{-30}\)). Nevertheless, this concept MSE  remains above the baseline value of \(0.1410\) (Appendix Fig.~\ref{fig:cub-performance}B).\looseness=-1

Taken together, these results confirm, on real-world data, that concept incompleteness is associated with endpoint and partial intervention harm, and that these are distinct failure modes. More importantly, however, our results strongly suggest that \(\mathcal L_{\mathrm{int}}\) reduces endpoint harm, whereas \(\mathcal L_{\mathrm{mono}}\) directly suppresses harmful partial interventions, suggesting both are required to encourage benign leakage.\looseness=-1

\section{Discussion and Conclusion}
\label{sec:discussion}


\mypara{Connecting Theoretical and Empirical Results.}
Our results establish two broader findings. First, we prove that sufficiency and localisation can hold even when interventions produce arbitrarily large intervention harm.
Second, we empirically complement these results by showing, in both synthetic and real-world tasks, that preventing endpoint harm via \(\mathcal L_{\mathrm{int}}\) is insufficient to ensure safe partial interventions. To address this,
we introduce \(\mathcal L_{\mathrm{mono}}\), a loss that mitigates partial intervention harm by minimising the CM's loss across distinct intervention trajectories. 
Contrary to previous intervention-aware losses that simulate training interventions~\citep{espinosa2023learning}, our loss 
avoids updating the concept representations (thereby focusing on adapting the label predictor to reduce partial harm), and does not require training an expensive sampling model.
Importantly, our results show that $\mathcal{L}_\text{int}$ and $\mathcal{L}_\text{mono}$ capture complementary aspects of harm: while $\mathcal{L}_\text{int}$ reduces harm from total concept corrections, $\mathcal{L}_\text{mono}$ reduces the prevalence and magnitude of harmful \textit{intervention paths}.\looseness=-1

\mypara{Partial Interventions and Out-of-Distribution Concept States.}
Sufficiency and localisation constrain how task-relevant information is represented, but not how the CM's task predictor responds to modifications of its input representations. This distinction is key in concept incompleteness, where the CM must rely on information represented outside the supervised concept set to maintain task performance. There, a CM may learn dependencies between the concept and residual representations that are disrupted when intervention changes $\widehat{\mathbf{c}}$ without updating $\mathbf{z}$, driving the joint representations $(\widehat{\ck}, \mathbf{z})$ out-of-distribution (OOD).
Crucially, such OOD intermediate states can occur even in completeness: intervening on a single concept without intervening on correlated or dependent concepts 
could generate an OOD concept vector $\widehat{\ck}$ and hamper the predictor's performance. This distributional shift is key to understanding why intervention harm is more likely as incompleteness increases, and why $\mathcal{L}_\text{int}$ struggles to eradicate partial intervention harm. As such, $\mathcal{L}_\text{mono}$ provides a mechanism for a CM to learn to better deal with potential OOD intermediate latent spaces, thereby significantly reducing partial intervention harm.\looseness=-1

Lastly, our results suggest a connection between intervention harm and shortcut learning~\citep{geirhos2020shortcut, marconato2023not}. 
If interventions increase task loss, this suggests the predictor has learned dependencies on the \emph{particular} joint distribution of predicted concepts and residual information, which may include errors in $\widehat{\mathbf{c}}$. This behaviour is consistent with shortcut learning where the predictor may exploit potentially spurious correlations between $\widehat{\mathbf{c}}$ and $\mathbf{z}$ (or within $\widehat{\mathbf{c}}$) that are predictive on the training distribution, but brittle under intervention. We hypothesise that $\mathcal{L}_{\mathrm{mono}}$ mitigates this behaviour by exposing the predictor during training to diverse combinations of correct and erroneous concept predictions, thereby reducing its reliance on brittle dependencies and encouraging predictions that remain robust across interventions.\looseness=-1

\mypara{Limitations and Future Work.}
We bring forth three limitations of our work. First, intervention trajectories in $\mathcal{L}_\text{mono}$ are randomly sampled, generating a disconnect between how we train them and how they are used afterwards (where they may be intervened according to some policy).
This motivates future work in exploring harm and monotonic training under varying policies.
Second, some partial intervention trajectories may result in implausible (i.e., OOD) intermediate states that could cause partial harm, even if full interventions would not be harmful.
Hence, future work could bias sampling in $\mathcal{L}_\text{mono}$ toward plausible states. 
Finally, sampling paths with equal probability may fail to cover distinct but meaningful regions of the intervention lattice. Therefore, future research can explore how to more efficiently sample paths that cover most important regions of the intervention lattice.\looseness=-1

\mypara{Conclusion.}
In this paper, we argued that intervention harm, and in particular partial intervention harm, is an understudied failure mode of CMs, directly undermining a key use case of these models. We proved that \textit{sufficiency} and \textit{localisation}, criteria previously suggested to encourage so-called \textit{benign leakage}, cannot guarantee that concept interventions improve downstream performance, and showed, across synthetic and real-world tasks, that concept-set incompleteness is associated with both endpoint and partial intervention harm. Moreover, we found that $\mathcal{L}_\text{int}$, a previous objective to encourage benign leakage, does not reliably prevent harmful partial corrections. To address this, we introduced $\mathcal{L}_\text{mono}$, a complementary objective that
substantially reduces the prevalence and magnitude of partial intervention harm while largely preserving concept and task performance. Overall, our results suggest that whether leakage is benign should be judged not only by the fully intervened endpoint, but also by considering all partial intervention paths an overseer may take.



\subsection*{AI use statement}
In this work, we used generative AI tools to: help develop theoretical models and conceptual frameworks, refine hypotheses, design and provide feedback on research methodology and experiments, support interpretation of results, and to assist in the formulation and writing of mathematical proofs. We did not use generative AI tools to generate synthetic data sets, implement methods, provide critical ingredients for proving mathematical claims, clean or reformat datasets, perform qualitative or thematic data analysis or assist with translation. Additionally, we used generative AI tools to create and edit software code and draft parts of the manuscript. We have reviewed all AI-assisted work substantially: we checked LLM-generated research ideas for potential plagiarism through a manual literature survey, LLM-generated code was verified and tested for correctness by multiple authors, all mathematical proofs and claims were reviewed and corrected by multiple authors and all LLM-generated text was reviewed and edited by multiple authors. We take responsibility for the final content of this work, including text, claims or artefacts produced with the aid of generative AI.

\subsection*{Ethics Statement}
This work investigates the safety and reliability of human intervention in concept-based models. The experiments use synthetic data and the established \texttt{CUB} benchmark and do not involve human participants, personal data, or sensitive attributes. No new datasets are collected or released. We have considered the potential for both beneficial and unintended uses of the proposed methodology and believe that explicitly characterising intervention harm contributes to the responsible development and evaluation of human-interpretable AI systems.

\subsection*{Reproducibility Statement}
We provide the assumptions and complete proofs of our theoretical results in the Appendix. Detailed experimental methods are provided for both the synthetic and \texttt{CUB} experiments, including data generation and preprocessing, dataset partitions, model architectures, training objectives, optimisation procedures, hyperparameters, intervention procedures, experimental designs, evaluation metrics, and statistical analyses. Random seeds and the procedures used to construct and match experimental conditions are specified where applicable. We additionally provide anonymised source code sufficient to reproduce the experiments, analyses, and figures reported in this work as supplementary material.

\section*{Acknowledgments}
CRSB is supported by a KCL AI+ Senior Fellowship. GV is supported by DSIT.

\bibliography{iclr2026_conference}
\bibliographystyle{iclr2026_conference}

\appendix
\setcounter{theorem}{0}
\section{Proof of Theorem \ref{thm:noisy-xor-correction-harm}}
We first restate the theorem more expansively:
\begin{theorem}[Interventions can be arbitrarily harmful under approximate benign leakage]
\label{app_thm:noisy-xor-correction-harm}

For every \(p\in(0,1/2)\) and every \(n>0\), there exists a finite data-generating process, concept representations $\widehat \ck \sim \widehat C$, residuals $\mathbf{z} \sim Z$ such that:
\begin{enumerate}
    \item $\widehat C_i$ is misaligned with ground-truth concept $C_i$ with probability \(p\). 
    \item 
    Localisation holds exactly
but sufficiency discrepancy is $\Delta_{\mathrm S}=\mathfrak{h}(p)$, where \(
        \mathfrak{h}(p)
        :=
        -p\log p-(1-p)\log(1-p)
    \) 
is the binary entropy. 
\item  The  task prediction of the optimal task classifier $\head$ is incorrect with probability \(p\).
\item Intervening by replacing \(\widehat C\) by its ground-truth value \(C\), while holding \(Z\) fixed, strictly increases expected squared loss:
    \begin{equation}
        \E\!\left[
            \bigl(
                Y-\thead(C,Z)
            \bigr)^2
        \right]
        -
        \E\!\left[
            \bigl(
                Y-\thead(\widehat C,Z)
            \bigr)^2
        \right]
        \geq
        np
        >
        0.
        \label{eq:noisy-xor-positive-harm}
    \end{equation}
\end{enumerate}

Equivalently, this implies that, for every \(\varepsilon>0\) and every \(M>0\), there exists an example that satisfies 
\begin{equation}
    \Delta_{\mathrm S}<\varepsilon,
    \qquad
    \Delta_{\mathrm L,1}=0,
    \qquad
    \eta(\mathcal{C})>M.
    \label{eq:noisy-xor-arbitrarily-benign}
\end{equation}
\end{theorem}
    
\label{proof_thm_1}
\begin{proof}
To prove the theorem, it suffices to construct a finite generating process, representation and task head such that the stated conditions hold.

A random variable \(B\) is called a \emph{fair coin} if
\[
    B\sim\operatorname{Bernoulli}(1/2).
\]
Let \(C\) and \(U\) be independent fair coins, and let
\[
    E\sim\operatorname{Bernoulli}(p),
    \qquad p\in(0,1/2),
\]
be independent of \((C,U)\). Here, \(C\) is the known concept, \(U\) is an
omitted task-relevant factor, and $E=1[C\neq \hat C]$ indicates whether the learned concept
predictor makes an error. Fix any \(n>0\) and set
\[
    b:=\sqrt{n+1}>1.
\]
Define
\begin{equation}
    X=(C,U,E),
    \qquad
    Y=C+bU.
    \label{eq:noisy-xor-data}
\end{equation}
The four possible pairs \((C,U)\) give the four distinct values
\[
    (0,0)\mapsto0,\qquad
    (1,0)\mapsto1,\qquad
    (0,1)\mapsto b,\qquad
    (1,1)\mapsto b+1.
\]
Because \(b>1\), these values are distinct. Hence the optimal predictor $\head^* : \mathcal C \times \mathcal U \to \mathcal Y$ 
 is a bijection from \(\{0,1\}^2\) onto
\(\{0,1,b,b+1\}\), so both \(C\) and \(U\) are deterministic functions of
\(Y\).

We assume the residual satisfies $Z=\widehat C\oplus U$ where \(\oplus\) denotes XOR (addition mod 2) so that
\begin{equation}
    \widehat C=C\oplus E,
    \qquad
    Z=\widehat C\oplus U.
    \label{eq:noisy-xor-representation}
\end{equation}

Notice \(\widehat C\) is a fair coin. Indeed, for every
\(a\in\{0,1\}\),
\begin{align}
    \Pr(\widehat C=a)
    &=
    \sum_{e\in\{0,1\}}
    \Pr(E=e)\Pr(C=a\oplus e)
    \nonumber=
    \frac12\sum_{e\in\{0,1\}}\Pr(E=e)
    =
    \frac12.
    \label{eq:noisy-xor-predicted-concept-fair}
\end{align}

In addition, \(\widehat C\) and \(Z\) are independent.
Indeed, because \(U\) is a fair coin independent of \(\widehat C\), since the latter is a deterministic function of random variables independent of $U$, for any
\(a,z\in\{0,1\}\),
\begin{align}
    \Pr(\widehat C=a,Z=z)
    &=
    \Pr\bigl(
        \widehat C=a,\,
        \widehat C\oplus U=z
    \bigr)
    \nonumber\\
    &=
    \Pr\bigl(
        \widehat C=a,\,
        U=a\oplus z
    \bigr)
    \nonumber\\
    &=
    \Pr(\widehat C=a)\Pr(U=a\oplus z)
    \nonumber\\
    &=
    \Pr(\widehat C=a)\frac12.
\end{align}
Moreover,
\[
    \Pr(Z=z)
    =
    \sum_{a\in\{0,1\}}\Pr(\widehat C=a,Z=z)
    =
    \frac12.
\]
Hence
\[
    \Pr(\widehat C=a,Z=z)
    =
    \Pr(\widehat C=a)\Pr(Z=z),
\]
which proves
\begin{equation}
    \widehat C\perp Z.
    \label{eq:noisy-xor-concept-residual-independence}
\end{equation}
In particular,
\begin{equation}
    H(\widehat C\mid Z)
    =
    H(\widehat C)
    =
    \log 2.
    \label{eq:noisy-xor-concept-residual-entropy}
\end{equation}

For localisation, recall that \(U\) is a deterministic function of \(Y\).
Since
\[
    \widehat C=Z\oplus U,
\]
\(\widehat C\) is therefore a deterministic function of \((Y,Z)\). Thus
\[
    H(\widehat C\mid Y,Z)=0.
\]
Combining this with the independence established above gives
\begin{align}
    I(Y;\widehat C\mid Z)
    &=
    H(\widehat C\mid Z)
    -
    H(\widehat C\mid Y,Z)
    \nonumber\\
    &=
    \log 2.
    \label{eq:noisy-xor-representation-localisation}
\end{align}
Similarly, because \(C\) is a deterministic function of \(Y\),
\[
    H(C\mid Y)=0,
\]
and hence
\begin{equation}
    I(Y;C)
    =
    H(C)-H(C\mid Y)
    =
    H(C)
    =
    \log 2.
    \label{eq:noisy-xor-ground-truth-localisation}
\end{equation}
Therefore
\begin{equation}
    \Delta_{\mathrm L,1}
    =
    \left|
        I(Y;C)-I(Y;\widehat C\mid Z)
    \right|
    =
    0.
    \label{eq:noisy-xor-exact-localisation}
\end{equation}
Thus localisation holds exactly.

For sufficiency, 
note 
$$ I(Y;Z \mid C) = H(Y \mid C) - H(Y \mid Z,C) = \log 2 -\mathfrak{h}(p), $$
where
\[
    \mathfrak{h}(p)
    :=
    -p\log p-(1-p)\log(1-p)
\]
is the binary entropy,
since 
$$ H(Y \mid C) = H(U) = \log 2$$
and 
$$H(Y \mid Z,C)=H(U \mid Z) = H(U \mid U \oplus E) = \mathfrak{h}(p).  $$
On the other hand, 
$$ I(Y;X\mid  C) = H(U \mid C) - H(U \mid X,C) = \log 2-0.$$
 Thus the sufficiency condition holds:
\begin{align}
    \Delta_{\mathrm S}
    &:=
    I(Y;X\mid  C)
    -
    I(Y;Z\mid C)
    \nonumber\\
    &=
    \mathfrak{h}(p).
    \label{eq:noisy-xor-sufficiency-gap}
\end{align}

We now analyse the risk and intervention harm of the optimal predictor.
Under square loss, the optimal predictor $\mathcal C \times \mathcal Z \to \mathcal Y$ is the conditional expectation
$$ \head^* (c, z) = \mathbb E[Y \mid \hat C = c, Z=z ] =  c(1-2p)+p+b(z\oplus c).$$
Indeed, $Y=C+bU$, and
$$ \mathbb E[U \mid \hat C = c, Z=z ] = \mathbb E[Z \oplus \hat C \mid \hat C = c, Z=z ] =  z \oplus c,$$
while using the fact that $C \perp Z$ (which follows from essentially the same proof as that of $\hat C \perp Z$),  we find
$$  \mathbb E[C \mid \hat C = c, Z=z ]
=\mathbb E[c\oplus E \mid \hat C = c] = p+c(1-2p).$$

Notice the optimal predictor $\head^* $ outputs values in $\mathcal Y^* =\{p,1-p,p+b,b+1-p\}$ which do not belong to $\mathcal Y = \{0,1,b,b+1 \} $ as $p\in (0,1/2)$. 
By matching the output of $\head^*$ to the closest element in $\mathcal Y$, we obtain 
the task label predictor
\begin{equation}
    \head(c,z)
    :=
    c+b(c\oplus z).
    \label{eq:noisy-xor-head}
\end{equation}
In other words, $\head$ is the projection of the optimal predictor $\head^*$ onto $\mathcal Y$, obtained by identifying $\mathcal Y^* \sim \mathcal Y$,  which is used to predict labels.

Now, since $\widehat C \oplus Z =U$,
\[
    \thead(\widehat C,Z)=Y
    \quad\Longleftrightarrow\quad
    \widehat C=C
    \quad\Longleftrightarrow\quad
    E=0,
\]
and both the concept prediction and the  task prediction are
incorrect with probability \(p\).

In addition,
\begin{equation}
    \thead(\widehat C,Z)-Y
    =
    \widehat C-C.
    \label{eq:noisy-xor-observational-error}
\end{equation}
Because \(C,\widehat C\in\{0,1\}\), this difference has squared value \(1\)
exactly when \(\widehat C\neq C\), or equivalently when \(E=1\). Therefore
\begin{equation}
    \bigl(
        \thead(\widehat C,Z)-Y
    \bigr)^2
    =
    \mathds{1}_{\{E=1\}}
    =
    E.
    \label{eq:noisy-xor-observational-squared-error}
\end{equation}
It follows that the squared-error risk is
\begin{equation}
    R_{\mathrm{obs}}
    :=
    \E\!\left[
        \bigl(
            Y-\thead(\widehat C,Z)
        \bigr)^2
    \right]
    =
    \E[E]
    =
    p.
    \label{eq:noisy-xor-observational-risk}
\end{equation}

We next verify the benign-leakage criteria.

We now perform the usual known-concept-only concept correction
\begin{equation}
    (\widehat C,Z)\longmapsto(C,Z).
    \label{eq:noisy-xor-coordinate-correction}
\end{equation}
This intervention replaces the predicted concept by its ground-truth value
but leaves the realised residual \(Z\) exactly unchanged. In particular,
\[
    Z_{\mathrm{after}}
    =
    Z_{\mathrm{before}}
    =
    \widehat C\oplus U
    =
    C\oplus E\oplus U.
\]
The input supplied to the task head therefore changes as follows:
\begin{equation}
    \underbrace{
        \thead(\widehat C,Z)
    }_{
        \widehat C+b(\widehat C\oplus Z)
    }
    \quad\longmapsto\quad
    \underbrace{
        \thead(C,Z)
    }_{
        C+b(C\oplus Z)
    }.
    \label{eq:noisy-xor-head-intervention}
\end{equation}
Although \(Z\) itself is unchanged, its decoding changes because the task head
now XORs \(Z\) with \(C\) rather than with \(\widehat C\). Specifically,
\begin{align}
    C\oplus Z
    &=
    C\oplus(C\oplus E\oplus U)
    \nonumber\\
    &=
    E\oplus U.
    \label{eq:noisy-xor-stale-decoding}
\end{align}
Thus the corrected prediction is
\begin{equation}
    \thead(C,Z)
    =
    C+b(E\oplus U).
    \label{eq:noisy-xor-corrected-prediction}
\end{equation}

If \(E=0\), then \(\widehat C=C\), so the intervention changes neither the
concept coordinate nor the prediction. In this case,
\[
    E\oplus U=U,
    \qquad
    \thead(C,Z) = C + bU =Y.
\]
If \(E=1\), then the intervention correctly changes
\(\widehat C=1-C\) to \(C\), but the fixed residual remains
\[
    Z=\widehat C\oplus U.
\]
Consequently,
\[
    E\oplus U=1\oplus U=1-U,
\]
so the omitted factor decoded by the task head changes from \(U\) to \(1-U\).
Therefore
\begin{equation}
    \thead(C,Z)-Y
    =
    \begin{cases}
        0,
        & E=0,\\[1mm]
        b(1-U)-bU=b(1-2U),
        & E=1.
    \end{cases}
    \label{eq:noisy-xor-corrected-error}
\end{equation}

Since \(U\in\{0,1\}\),
\[
    \bigl(1-2U\bigr)^2=1.
\]
It follows pointwise that
\begin{equation}
    \bigl(
        \thead(C,Z)-Y
    \bigr)^2
    =
    \begin{cases}
        0, & E=0,\\
        b^2, & E=1,
    \end{cases}
    =
    b^2E.
    \label{eq:noisy-xor-corrected-squared-error}
\end{equation}
Taking expectations now gives, explicitly,
\begin{align}
    R_{\mathrm{corr}}
    &:=
    \E\!\left[
        \bigl(
            Y-\thead(C,Z)
        \bigr)^2
    \right]
    \nonumber\\
    &=
    \E[b^2E]
    =
    b^2\Pr(E=1)
    =
    b^2p.
    \label{eq:noisy-xor-corrected-risk}
\end{align}
Here, it is worth highlighting how correcting a concept differs from improving the task head: before corrections, we had $\thead(\hat C,Z)-Y = \hat C - C$, so that better concept predictions result in better predictions overall, while after intervention we see from above that now $ \thead(C,Z)-Y$ only depends on $U$, and produces greater risk. Thus, intervention did not remove an error, but rather worsened the risk by relocating it. 

Combining this with \(R_{\mathrm{obs}}=p\), we obtain
\begin{equation}
    \eta(\mathcal{C})
    :=
    R_{\mathrm{corr}}-R_{\mathrm{obs}}
    =
    (b^2-1)p
    =
    np
    >
    0.
    \label{eq:noisy-xor-harm-calculation}
\end{equation}
The intervention harm is also positive if we use the optimal predictor $\head^*$ directly (rather than its projection $\head$), in which case $$ \eta(\mathcal{C}) = R_{\mathrm{corr}}-R_{\mathrm{obs}}=\left(p^2+b^2p \right)-\left( p-p^2\right) = 2p^2+np>np>0.$$

Finally, let \(\varepsilon>0\) and \(M>0\). Since
\(\mathfrak{h}\to0\) as \(p\downarrow0\), we may first choose
\(p\in(0,1/2)\) sufficiently small that
\[
    \Delta_{\mathrm S}
    =
    \mathfrak{h}
    <
    \varepsilon.
\]
We may then choose \(n>M/p\) and set \(b=\sqrt{n+1}\). For these choices,
\[
    \Delta_{\mathrm L,1}=0,
    \qquad
    \eta(\mathcal{C}) \geq np>M.
\]
Thus arbitrarily accurate satisfaction of the sufficiency criterion, together
with exact localisation, places no finite upper bound on the harm caused by
ground-truth concept correction.
\end{proof}

\section{Experiment Details}
\label{app:experimental_details}

\subsection{Synthetic Experiments}
\label{app:synthetic-methods}

\subsubsection{Data-Generating Process}

We adapted the known--unknown concept construction introduced by \citet{parisini2026concept}. Each synthetic observation contained \(k_K=5\) supervised concepts and \(k^\prime\) additional unobserved concepts:
\[
    \boldsymbol{\chi}
    =
    (\ck ,\cu )
    \in\mathbb R^{5+k^\prime}.
\]
The complete concept vector was sampled as
\[
    \boldsymbol{\chi}\sim
    \mathcal N(\mathbf 0,\mathbf B).
\]

We evaluated three covariance conditions. In the \emph{diagonal} condition, \(\mathbf B=\mathbf I\). To construct the other two conditions, we generated a random symmetric positive-definite correlation matrix. The \emph{correlated} condition used this complete matrix. The \emph{block-independent} condition retained its within-known and within-unknown covariance blocks but set the cross-covariance blocks to zero. The covariance construction was matched across these conditions for each seed and value of \(k^\prime\).

Inputs were generated through a noisy linear map:
\[
    \mathbf x
    =
    \mathbf A\boldsymbol{\chi}
    +\boldsymbol{\epsilon}_x,
    \qquad
    \boldsymbol{\epsilon}_x
    \sim\mathcal N(\mathbf 0,\sigma_x^2\mathbf I),
\]
where \(\mathbf A\in\mathbb R^{64\times(5+k^\prime)}\) had independently sampled Gaussian entries and we set $\sigma_x=4$.

A Gaussian task-weight vector \(\boldsymbol{\psi}=(\boldsymbol{\psi}_K,\boldsymbol{\psi}_U)\) was sampled once for each synthetic world. For a specified incompleteness parameter \(\omega\), the known and unknown components were scaled as
\[
    \widetilde{\boldsymbol{\psi}}_K
    =(1-\omega)\boldsymbol{\psi}_K,
    \qquad
    \widetilde{\boldsymbol{\psi}}_U
    =\omega\boldsymbol{\psi}_U.
\]
Binary labels were then generated by
\[
    y
    =
    \mathbbm 1
    \left\{
        \widetilde{\boldsymbol{\psi}}_K^\top\ck 
        +
        \widetilde{\boldsymbol{\psi}}_U^\top\cu 
        +
        \epsilon_y>0
    \right\},
    \qquad
    \epsilon_y\sim\mathcal N(0,\sigma_y^2).
\]
We set $\sigma_y=4$. Thus, \(\omega\) controls the relative weighting of known and unknown task
directions.

Each synthetic world contained \(N=8{,}000\) observations. A seeded random
split assigned 75\% to development and 25\% to held-out testing. The
development observations were further divided into 80\% training and 20\%
validation subsets, giving \(4{,}800/1{,}200/2{,}000\) observations.

\subsubsection{Hybrid CBM and Interventions}

The hybrid CBM contained a linear concept encoder
\[
    \widehat{\ck}
    =g_\phi(\mathbf x)\in\mathbb R^5,
\]
a five-dimensional linear residual encoder
\[
    \mathbf z=h_\psi(\mathbf x)\in\mathbb R^5,
\]
and a linear binary task head
\[
    y=f_\theta(
        [\widehat{\ck},\mathbf z]).
\]
Continuous synthetic concepts were supervised using mean-squared error, and
binary task prediction used binary cross-entropy with logits.

For intervention set \(S\subseteq\{1,\ldots,5\}\), we constructed
\[
    \widehat c_j^{\,S}
    =
    \begin{cases}
        c_j, & j\in S,\\
        \widehat c_j, & j\notin S.
    \end{cases}
\]
The residual representation was held fixed. After training,
each intervention state \(S\) was evaluated on the same held-out test set. We defined its empirical test risk as the mean task loss
\[
    \mathcal{R}_S
    =
    \frac{1}{n_{\mathrm{test}}}
    \sum_i
    \ell\!\left(
        f_\theta(
            [\widehat{\ck}_{i}^{\,S},
             \mathbf z_i]),
        y_i
    \right),
\]
and intervention harm was
\[
    \eta(S)=\mathcal{R}_S-\mathcal{R}_\varnothing.
\]

\subsubsection{Training Objective}

Models were trained using
\[
    \mathcal L
    =
    \frac{
        \mathcal L_{\mathrm{task}}
        +\lambda_c\mathcal L_{\mathrm{concept}}
        +\lambda_{\mathrm{int}}\mathcal L_{\mathrm{int}}
        +\lambda_{\mathrm{mono}}\mathcal L_{\mathrm{mono}}
    }{
        1+\lambda_c+\lambda_{\mathrm{int}}+\lambda_{\mathrm{mono}}
    },
    \qquad \lambda_c=1.
\]
The endpoint loss \(\mathcal L_{\mathrm{int}}\) was the task loss after all
five concept predictions had been replaced by ground truth.

For an intervention edge
\(S\rightarrow S\cup\{j\}\), the monotonic penalty was
\[
    \left[
        \ell_{S\cup\{j\}}-\ell_S
    \right]_+.
\]
$m_{\mathrm{mono}}$ edges were sampled per minibatch. For each edge, intervention depth was
sampled uniformly from \(\{0,\ldots,4\}\), followed by a random permutation of the five concepts. The first concepts in this permutation defined \(S\), and the next concept defined \(j\). 

When stop-gradient was enabled, the predicted concept representation was detached only within \(\mathcal L_{\mathrm{mono}}\). The monotonic loss could therefore update the residual encoder and task head, but could not directly update the concept encoder. The ordinary concept and task losses continued to train the concept encoder.

Models were optimised using Adam with learning rate \(10^{-3}\), batch size 600, and at most 300 epochs. Training continued for at least 50 epochs and stopped after 30 validation epochs without an improvement of \(10^{-5}\).
The checkpoint with the lowest complete validation objective was retained.

\subsubsection{Experimental Design}

We crossed
\[
\begin{aligned}
    k^\prime&\in\{2,5,10,15\},\\
    \omega&\in\{0.10,0.25,0.50,0.75,0.90\},\\
    \lambda_{\mathrm{int}}&\in\{0,1,10,100,1000\},\\
    \lambda_{\mathrm{mono}}&\in\{0,10,100,1000\},
\end{aligned}
\]
with the three covariance conditions and seeds \(\{0,1,2,3,4\}\) and set $m_{\mathrm{mono}}=5$. Both stop-gradient settings were evaluated whenever \(\lambda_{\mathrm{mono}}>0\); one canonical setting was retained when \(\lambda_{\mathrm{mono}}=0\). The resulting factorial contained \(10{,}500\) models.

To explore the impact of varying $m_\mathrm{mono}$ we fixed \(\lambda_{\mathrm{int}}=10\), \(\lambda_{\mathrm{mono}}=100\), the block-independent covariance condition, and stop-gradient enabled. We varied
\[
    m_{\mathrm{mono}}\in\{1,2,5,10,20,40\}
\]
for \(k^\prime\in\{5,10\}\), all five values of \(\omega\), and five seeds,
yielding 300 additional models.

\subsubsection{Evaluation and Statistical Analysis}

All \(2^5=32\) intervention states and all \(5!=120\) correction orderings were evaluated. Endpoint harm was
\[
    \eta(\mathcal{C})
    =
    \eta(\{1,\ldots,5\}),
\]
and maximum partial harm was
\[
    H_{\mathrm{partial}}^{\max}
    =
    \max_{\varnothing\subset S\subset\{1,\ldots,5\}}\eta(S).
\]

For adjusted analyses, outcomes and continuous predictors were standardised. Regularisation weights were transformed as \(\log_{10}(1+\lambda)\). Additive ordinary least-squares models included \(k^\prime\), \(\omega\), \(\log_{10}(1+\lambda_{\mathrm{int}})\), \(\log_{10}(1+\lambda_{\mathrm{mono}})\), and covariance condition. Standard errors were clustered by synthetic world, defined by seed, \(k^\prime\), and covariance condition. The stop-gradient contrast was estimated only among models with \(\lambda_{\mathrm{mono}}>0\).

\subsubsection{Robustness and ablation analyses for synthetic experiments}
\label{app:synthetic-robustness}

We performed additional analyses to assess the robustness of the synthetic results and the effects of key modelling choices. Multivariable analyses confirmed the complementary effects of the two intervention objectives: increasing $\lambda_{\mathrm{int}}$ was primarily associated with reduced
endpoint harm, whereas increasing $\lambda_{\mathrm{mono}}$ was primarily associated with reduced maximum partial harm
(Fig.~\ref{fig:synthetic-adjusted}). Importantly, these reductions in harm were achieved with little change in predictive performance: across moderate regularisation weights ($\lambda_{\mathrm{int}}\in\{1,10,100\}$ and
$\lambda_{\mathrm{mono}}\in\{10,100\}$), task accuracy differed by at most 0.043 percentage points from the unregularised model. The effects on intervention harm were also robust across concept covariance structures (Fig.~\ref{fig:synthetic-robustness}B--D). Stopping gradients from $\mathcal{L}_{\mathrm{mono}}$ through the concept encoder mitigated a significant increase in concept prediction error (Fig.~\ref{fig:synthetic-robustness}A), while the benefit of sampling additional monotonic edges largely saturated at $m_{\mathrm{mono}}=5$
(Fig.~\ref{fig:synthetic-edges}).

\begin{figure}[h!]
    \centering
        \centering
        \includegraphics[width=0.7\linewidth]{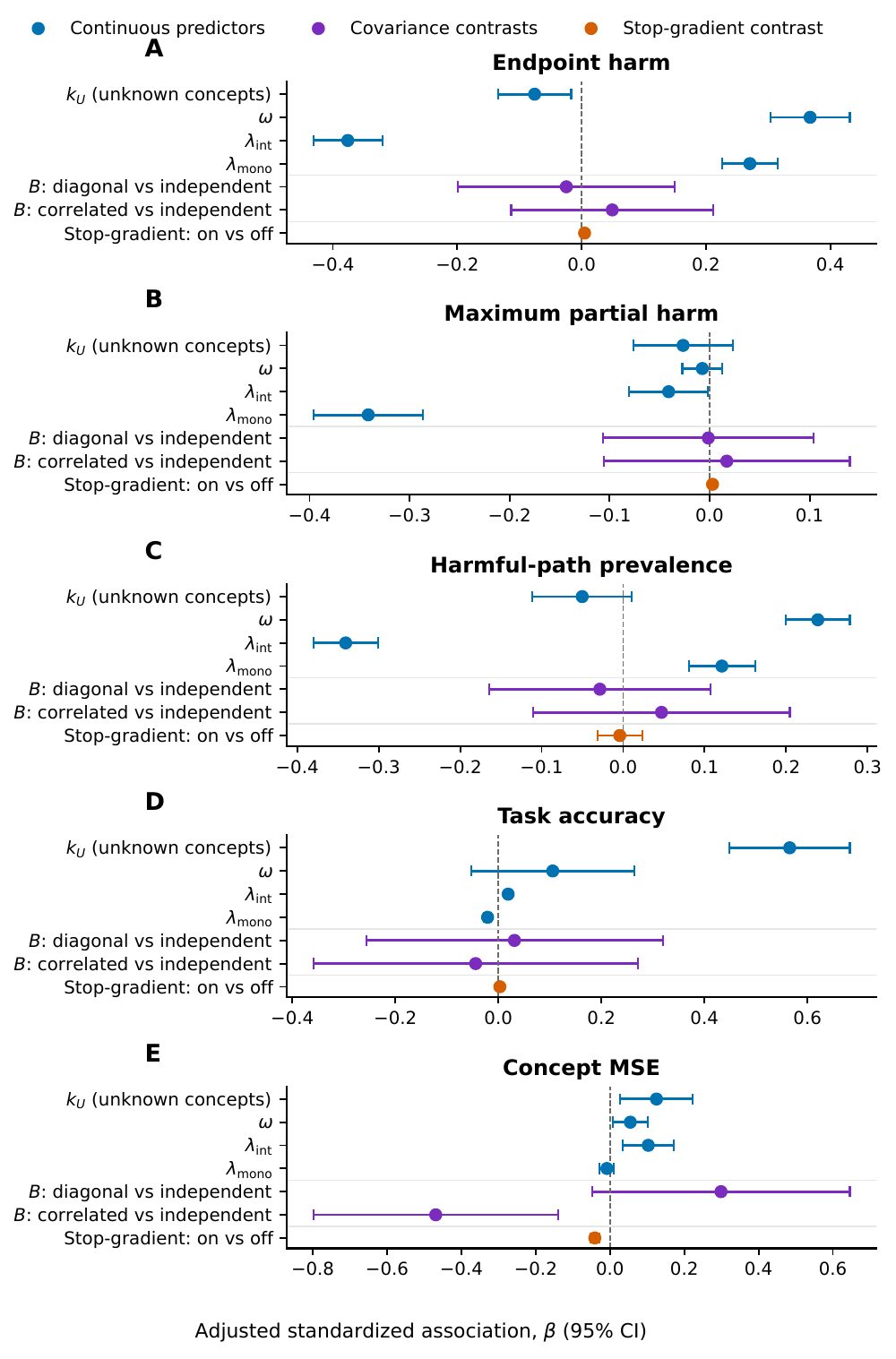}
        \caption{Adjusted determinants of intervention harm and model performance.
        Adjusted standardised associations with \textbf{(A)} endpoint harm; \textbf{(B)} maximum partial harm; \textbf{(C)} harmful-path prevalence; \textbf{(D)} task accuracy; and \textbf{(E)} concept mean-squared error. Points show regression coefficients and bars show 95\% cluster-robust confidence intervals across 60 matched synthetic worlds. Outcomes and continuous predictors were standardised. Covariance contrasts use the block-independent construction as the reference; the stop-gradient contrast was estimated among models with \(\lambda_{\mathrm{mono}}>0\).}
        \label{fig:synthetic-adjusted}
\end{figure}

\begin{figure}[h!]
    \centering
        \centering
        \includegraphics[width=0.70\linewidth]
            {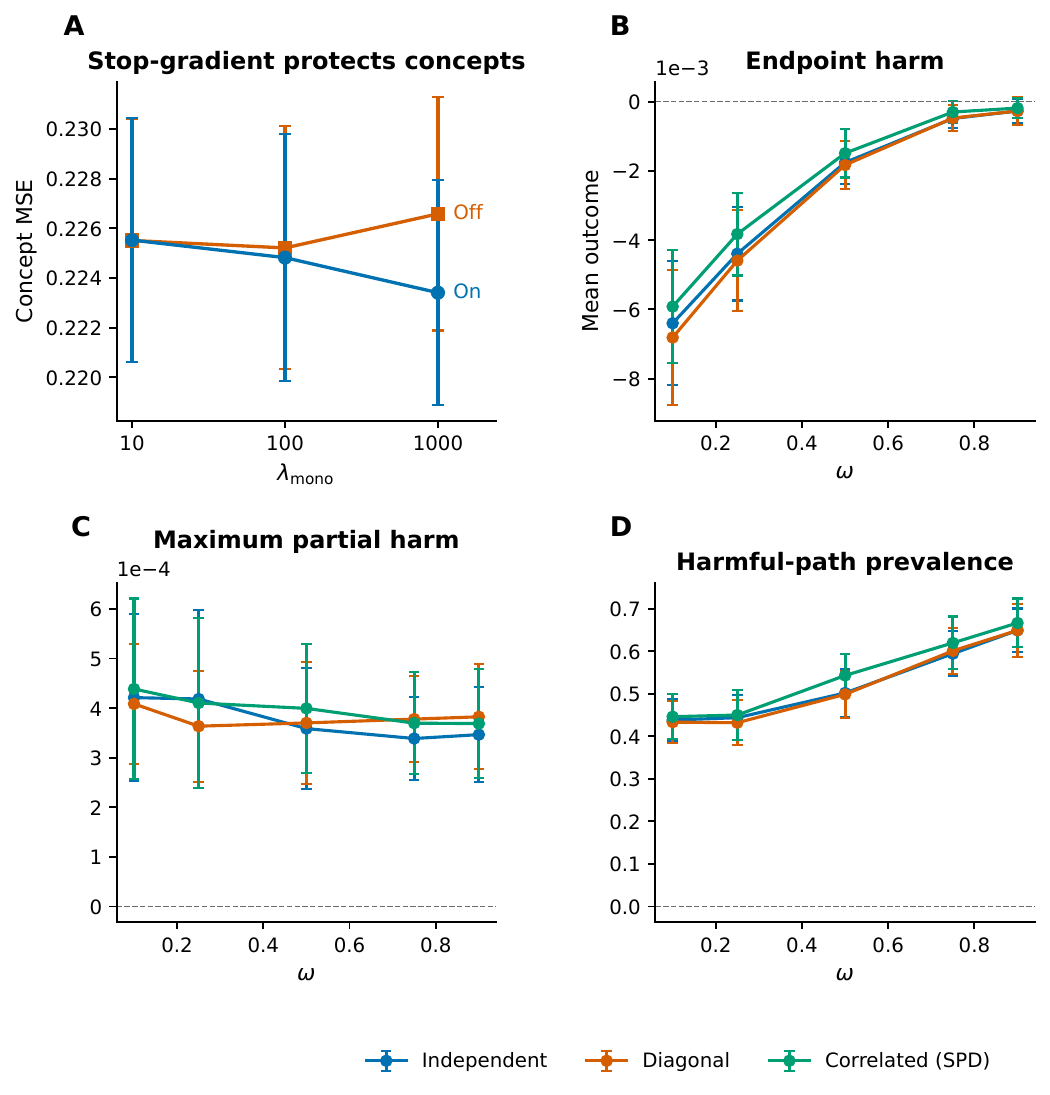}
        \caption{Robustness to stop-gradient and concept covariance structure.
        \textbf{(A)} Concept mean-squared error with and without stopping gradients from \(\mathcal{L}_{\mathrm{mono}}\) through the concept encoder. Values were first averaged within each of 60 matched synthetic worlds over \(\lambda_{\mathrm{int}}\) and \(\omega\); points and bars show means and 95\% confidence intervals across worlds. \textbf{(B--D)} Endpoint harm, maximum partial harm, and harmful-path prevalence as functions of \(\omega\) under block-independent, diagonal, and correlated positive-definite covariance structures. Stop-gradient-enabled models were averaged over the regularisation grid within each matched seed--\(k^\prime\) unit. Points and bars show means and 95\% confidence intervals across 20 units per covariance condition and value of \(\omega\).}
        \label{fig:synthetic-robustness}
\end{figure}

\begin{figure}[h!]
    \centering
        \centering
        \includegraphics[width=\linewidth]
            {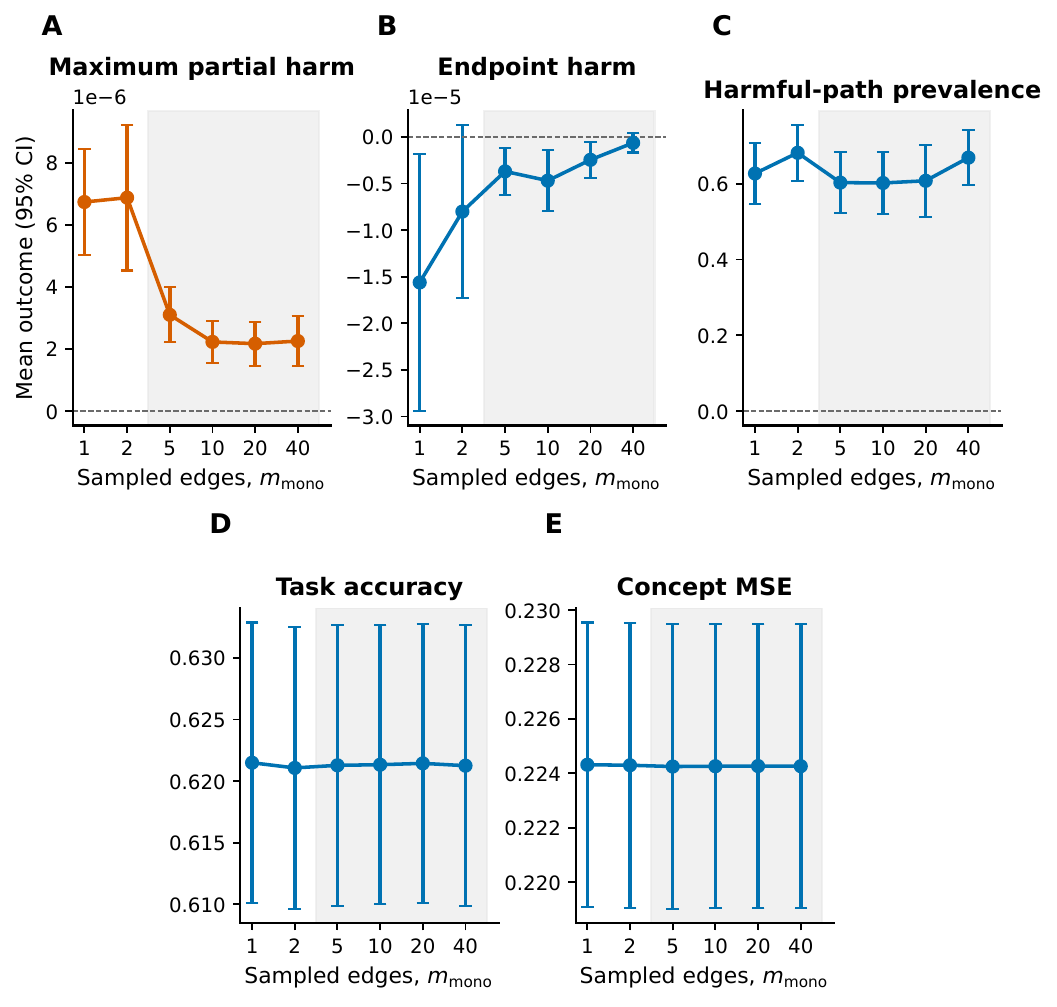}
        \caption{Effect of the number of monotonic edges sampled per minibatch.
        Mean \textbf{(A)} maximum partial harm; \textbf{(B)} endpoint harm; \textbf{(C)} harmful-path prevalence; \textbf{(D)} task accuracy; and \textbf{(E)} concept mean-squared error as the number of sampled monotonic edges \(m_{\mathrm{mono}}\) was varied. The ablation fixed \(\lambda_{\mathrm{int}}=10\), \(\lambda_{\mathrm{mono}}=100\), block-independent concept covariance, and stop-gradient. Points and bars show means and 95\% confidence intervals across 50 models spanning two values of \(k^\prime\), five values of \(\omega\), and five seeds. Shading denotes \(m_{\mathrm{mono}}\geq5\), beyond which the reduction in maximum partial harm largely plateaued.}
        \label{fig:synthetic-edges}
\end{figure}

\subsection{Real-world Data Experiments}
\label{app:cub-methods}

\subsubsection{Representations, Concepts, and Data Partition}

\texttt{CUB} contains \(11{,}788\) images from 200 bird species and 312 binary image-level attributes. We encoded each image using the pretrained CLIP ViT-B/16 image encoder and L2-normalised the resulting 512-dimensional embedding. CLIP parameters were fixed throughout.

We used one class-stratified split for the complete experiment: \(7{,}488\) images for training, \(1{,}926\) for validation, and \(2{,}374\) for held-out testing. Embeddings were standardised using feature-wise training-set means and standard deviations, with the same transformation subsequently applied to the validation and test data. Test images were not used for concept-set selection, training, or early stopping.

\subsubsection{Pre-model Concept Incompleteness}

To vary concept incompleteness while holding concept-set size fixed, we generated 210 candidate sets containing 50 attributes each. We first fitted a linear 200-class probe using all 312 attributes in the training partition and ranked attributes by the L2 norm of their class coefficients. This ranking was used only to diversify candidate generation.

We sampled 30 unique concept sets at 7 importance biases
\[
    b\in\{-6,-3,-1,0,1,3,6\}.
\]
Let \(r_j\in[0,1]\) denote the normalised importance rank of attribute \(j\), with zero denoting the most highly ranked attribute, we defined the attribute sampling weight as
\[
    p_j\propto\exp(-b r_j).
\]
Positive biases therefore favoured more task-important attributes, negative biases favoured less important attributes, and \(b=0\) produced uniform sampling.

For each candidate concept set \(\mathcal C\), we fitted a separate linear 200-class classifier using its 50 ground-truth concepts and defined pre-model incompleteness as
\[
    \Omega_{\mathcal C}^{\mathrm{pre}}
    =
    \frac{
        \widehat H(Y\mid C)
    }{
        \widehat H(Y)
    }.
\]
The numerator was the classifier's best validation cross-entropy and the denominator was the validation cross-entropy of a null classifier using fitting-set class frequencies. This is a normalised probe-based measure rather than a literal information-theoretic fraction; higher values indicate that the selected concepts leave more task-relevant information unexplained.

Candidates were divided into ten quantile strata according to \(\Omega_{\mathcal C}^{\mathrm{pre}}\), and eight sets were sampled without replacement from each stratum. The resulting 80 sets spanned \(\Omega_{\mathcal C}^{\mathrm{pre}}=0.597\)--\(0.863\), were fixed before hybrid-model training, and were shared across all three training regimes (Fig.~\ref{fig:cub-concept-selection}).

\begin{figure}[h!]
    \centering
        \centering
        \includegraphics[width=\linewidth]
            {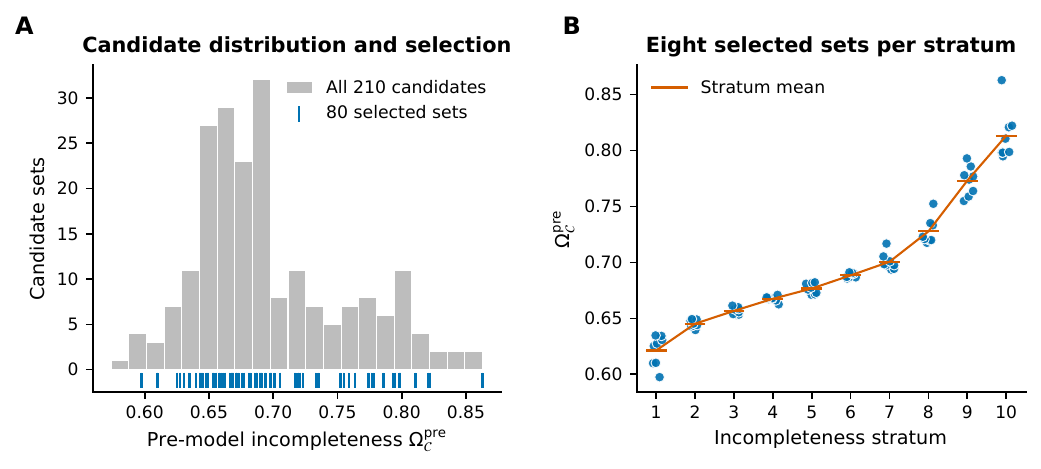}
        \caption{Construction of fixed-size concept sets spanning pre-model incompleteness. \textbf{(A)} Distribution of \(\Omega_{\mathcal C}^{\mathrm{pre}}\) across 210 candidate sets containing 50 concepts each. Blue marks identify the 80 sets selected for model training. \textbf{(B)} Pre-model incompleteness of the eight selected sets within each of ten quantile strata; orange lines show stratum means. Higher \(\Omega_{\mathcal C}^{\mathrm{pre}}\) indicates that more task-relevant information remained unexplained by the selected concepts.}
        \label{fig:cub-concept-selection}
\end{figure}

\subsubsection{Hybrid CBM and Matched Training}

For each concept set, the hybrid CBM contained a linear concept encoder, a 32-dimensional linear residual encoder, and a linear 200-class task head:
\[
    \widehat{\mathbf c}
    =
    \sigma(g_\phi(\mathbf x))\in[0,1]^{50},
    \qquad
    \mathbf z=h_\psi(\mathbf x)\in\mathbb R^{32},
\]
\[
    \mathbf y
    =
    f_\theta([\widehat{\mathbf c},\mathbf z]).
\]
Concept supervision used binary cross-entropy and observational task training used multiclass cross-entropy, with \(\lambda_c=1\).

For every concept set, we trained:
\[
\begin{array}{lcc}
    \toprule
    \text{Training regime}
        &\lambda_{\mathrm{int}}
        &\lambda_{\mathrm{mono}}\\
    \midrule
    \text{baseline} &0&0\\
    \text{endpoint} &1&0\\
    \text{endpoint + monotonic, stop-gradient} &1&100\\
    \bottomrule
\end{array}
\]
giving \(80\times3=240\) models. The three models for a concept set shared their initial parameters.

The normalised training objective was
Models were trained using
\[
    \mathcal L
    =
    \frac{
        \mathcal L_{\mathrm{task}}
        +\lambda_c\mathcal L_{\mathrm{concept}}
        +\lambda_{\mathrm{int}}\mathcal L_{\mathrm{int}}
        +\lambda_{\mathrm{mono}}\mathcal L_{\mathrm{mono}}
    }{
        1+\lambda_c+\lambda_{\mathrm{int}}+\lambda_{\mathrm{mono}}
    },
    \qquad \lambda_c=1.
\]
The endpoint loss was calculated after replacing all predicted concepts with their ground-truth values.

For monotonic training, we set $m_{\mathrm{mono}}=5$. Intervention depth was sampled uniformly from \(\{0,\ldots,49\}\), and a random concept permutation defined the current set \(S\) and next corrected concept \textit{i}. The monotonic penalty was
\[
    \left[
        \ell_{S\cup\{j\}}-\ell_S
    \right]_+.
\]
Gradients from this term were stopped through predicted concepts but retained through the residual encoder and task head.

Models were optimised using AdamW with learning rate \(10^{-3}\), weight decay \(10^{-5}\), and batch size 512. Training continued for at most 300 epochs, with a minimum of 50 epochs and patience of 30 validation epochs at an improvement threshold of \(10^{-5}\). The parameters minimising the full
validation objective were retained.

\subsubsection{Intervention Evaluation}

For an intervention set \(S\), selected predicted concepts were replaced by their image-level ground-truth values while the residual representation was held fixed. All three matched models for a concept set were evaluated using the same 512 random concept-correction orderings.

At state \(S\), test risk was
\[
    \mathcal{R}_S
    =
    -\frac{1}{n_{\mathrm{test}}}
    \sum_i
    \log p_\theta(
        y_i\mid\widehat{\mathbf c}^{\,S}_i,\mathbf z_i),
\]
and harm was \(\eta(S)=\mathcal{R}_S-\mathcal{R}_\varnothing\). We measured endpoint harm after all 50 concepts had been corrected, the maximum harm observed among the intermediate states of the 512 sampled paths, and the fraction of sampled paths containing at least one intermediate state with harm above \(10^{-8}\). The reported \texttt{CUB} maximum is therefore a maximum over sampled paths rather than all \(2^{50}\) possible intervention states.

We additionally recorded observational top-1 accuracy, fully intervened accuracy, and concept MSE. 

\subsubsection{Statistical Analysis}

For each outcome, the outcome and \(\Omega_{\mathcal C}^{\mathrm{pre}}\) were standardised across the complete 240-model design. We fitted
\[
    Y_z
    =
    \beta_0
    +\beta_\Omega
        (\Omega_{\mathcal C}^{\mathrm{pre}})
    +\beta_1\mathbbm 1(\mathrm{endpoint})
    +\beta_2\mathbbm 1(\mathrm{endpoint+monotonic})
    +\epsilon,
\]
with the baseline regime as reference. Standard errors were clustered by concept-set identifier, preserving the matched three-model design. This additive model estimates the average association with incompleteness and each intervention-training regime's adjusted difference from baseline. The results are shown in Fig. \ref{fig:cub-performance}.

\begin{figure}[h!]
    \centering
        \centering
        \includegraphics[width=\linewidth]
            {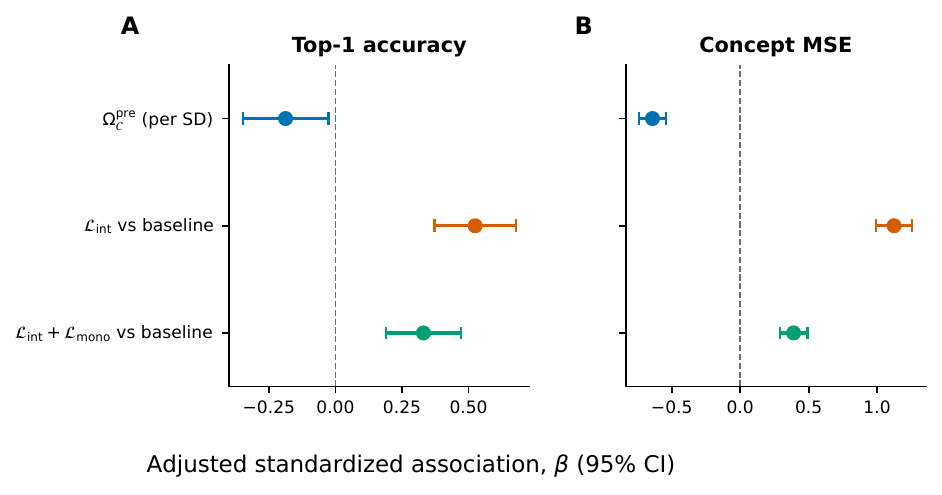}
        \caption{Adjusted associations with model performance and concept quality on \texttt{CUB}. Adjusted standardised associations with \textbf{(A)} top-1 accuracy and \textbf{(B)} concept mean-squared error. Each panel shows the association with a one-standard-deviation increase in pre-model incompleteness, \(\Omega_{\mathcal C}^{\mathrm{pre}}\), and the adjusted differences between endpoint training or endpoint-plus-monotonic training and the baseline regime. Points show regression coefficients and bars show 95\% cluster-robust confidence intervals; dashed vertical lines denote no association. Outcomes and \(\Omega_{\mathcal C}^{\mathrm{pre}}\) were standardised. Models included all 240 runs across 80 matched concept sets, with baseline as the regime reference and standard errors clustered by concept-set identifier.}
        \label{fig:cub-performance}
\end{figure}

\end{document}

%% file: math_commands.tex
\usepackage{amsmath,amsfonts,bm}

\def\eqref#1{equation~\ref{#1}}

\def\1{\bm{1}}

\DeclareMathAlphabet{\mathsfit}{\encodingdefault}{\sfdefault}{m}{sl}
\SetMathAlphabet{\mathsfit}{bold}{\encodingdefault}{\sfdefault}{bx}{n}

\newcommand{\E}{\mathbb{E}}



%% file: iclr2026_conference.bib
@inproceedings{koh2020concept,
  title={Concept bottleneck models},
  author={Koh, Pang Wei and Nguyen, Thao and Tang, Yew Siang and Mussmann, Stephen and Pierson, Emma and Kim, Been and Liang, Percy},
  booktitle={International conference on machine learning},
  pages={5338--5348},
  year={2020},
  organization={Pmlr}
}

@article{marconato2022glancenets,
  title={Glancenets: Interpretable, leak-proof concept-based models},
  author={Marconato, Emanuele and Passerini, Andrea and Teso, Stefano},
  journal={Advances in neural information processing systems},
  volume={35},
  pages={21212--21227},
  year={2022}
}

@article{steinmann2023learning,
  title={Learning to intervene on concept bottlenecks},
  author={Steinmann, David and Stammer, Wolfgang and Friedrich, Felix and Kersting, Kristian},
  journal={arXiv preprint arXiv:2308.13453},
  year={2023}
}

@article{pugnana2026deferring,
  title={Deferring concept bottleneck models: Learning to defer interventions to inaccurate experts},
  author={Pugnana, Andrea and Massidda, Riccardo and Giannini, Francesco and Barbiero, Pietro and Espinosa Zarlenga, Mateo and Pellungrini, Roberto and Dominici, Gabriele and Giannotti, Fosca and Bacciu, Davide},
  journal={Advances in Neural Information Processing Systems},
  volume={38},
  pages={67989--68044},
  year={2026}
}

@inproceedings{singhi2024improving,
  title={Improving intervention efficacy via concept realignment in concept bottleneck models},
  author={Singhi, Nishad and Kim, Jae Myung and Roth, Karsten and Akata, Zeynep},
  booktitle={European Conference on Computer Vision},
  pages={422--438},
  year={2024},
  organization={Springer}
}

@inproceedings{zarlenga2023towards,
  title={Towards robust metrics for concept representation evaluation},
  author={Espinosa Zarlenga, Mateo and Barbiero, Pietro and Shams, Zohreh and Kazhdan, Dmitry and Bhatt, Umang and Weller, Adrian and Jamnik, Mateja},
  booktitle={Proceedings of the AAAI Conference on Artificial Intelligence},
  volume={37},
  number={10},
  pages={11791--11799},
  year={2023}
}

@article{knab2026s,
  title={What’s in the bottle? a survey and roadmap of concept bottleneck models},
  author={Knab, Patrick and Steinmann, David and Bartelt, Christian and Kersting, Kristian and Schiele, Bernt and Seidl, Thomas and Schlegel, Udo and Stammer, Wolfgang},
  journal={Transactions on Machine Learning Research},
  year={2026}
}

@inproceedings{yang2023labo,
  title={Language in a bottle: Language model guided concept bottlenecks for interpretable image classification},
  author={Yang, Yue and Panagopoulou, Artemis and Zhou, Shenghao and Jin, Daniel and Callison-Burch, Chris and Yatskar, Mark},
  booktitle={2023 IEEE/CVF Conference on Computer Vision and Pattern Recognition (CVPR)},
  pages={19187--19197},
  year={2023},
  organization={IEEE}
}

@article{ragkousis2024tree,
  title={Tree-based leakage inspection and control in concept bottleneck models},
  author={Ragkousis, Angelos and Parbhoo, Sonali},
  journal={arXiv preprint arXiv:2410.06352},
  year={2024}
}

@article{zarlenga2025avoiding,
  title={Avoiding leakage poisoning: Concept interventions under distribution shifts},
  author={Espinosa Zarlenga, Mateo and Dominici, Gabriele and Barbiero, Pietro and Shams, Zohreh and Jamnik, Mateja},
  journal={arXiv preprint arXiv:2504.17921},
  year={2025}
}

@article{geirhos2020shortcut,
  title={Shortcut learning in deep neural networks},
  author={Geirhos, Robert and Jacobsen, J{\"o}rn-Henrik and Michaelis, Claudio and Zemel, Richard and Brendel, Wieland and Bethge, Matthias and Wichmann, Felix A},
  journal={Nature Machine Intelligence},
  volume={2},
  number={11},
  pages={665--673},
  year={2020},
  publisher={Nature Publishing Group UK London}
}

@article{marconato2023not,
  title={Not all neuro-symbolic concepts are created equal: Analysis and mitigation of reasoning shortcuts},
  author={Marconato, Emanuele and Teso, Stefano and Vergari, Antonio and Passerini, Andrea},
  journal={Advances in Neural Information Processing Systems},
  volume={36},
  pages={72507--72539},
  year={2023}
}

@inproceedings{rao2024discover,
  title={Discover-then-name: Task-agnostic concept bottlenecks via automated concept discovery},
  author={Rao, Sukrut and Mahajan, Sweta and B{\"o}hle, Moritz and Schiele, Bernt},
  booktitle={European Conference on Computer Vision},
  pages={444--461},
  year={2024},
  organization={Springer}
}

@article{skirzynski2026measuring,
  title={Measuring What Matters: Synthetic Benchmarks for Concept Bottleneck Models},
  author={Skirzynski, Julian and Cheon, Harry and Kadekodi, Shreyas and Stewart, Meredith and Ustun, Berk},
  journal={arXiv preprint arXiv:2606.04326},
  year={2026}
}

@inproceedings{ramaswamy2023overlooked,
  title={Overlooked factors in concept-based explanations: Dataset choice, concept learnability, and human capability},
  author={Ramaswamy, Vikram V and Kim, Sunnie SY and Fong, Ruth and Russakovsky, Olga},
  booktitle={2023 IEEE/CVF Conference on Computer Vision and Pattern Recognition (CVPR)},
  pages={10932--10941},
  year={2023},
  organization={IEEE}
}

@inproceedings{de2026mechanistic_cbm,
  title={Learning concept bottleneck models from mechanistic explanations},
  author={De Santis, Antonio and Tong, Schrasing and Brambilla, Marco and Kagal, Lalana},
  booktitle={The Fourteenth International Conference on Learning Representations},
  year={2026}
}

@article{vandenhirtz2024stochastic,
  title={Stochastic concept bottleneck models},
  author={Vandenhirtz, Moritz and Laguna, Sonia and Marcinkevi{\v{c}}s, Ri{\v{c}}ards and Vogt, Julia E},
  journal={Advances in Neural Information Processing Systems},
  volume={37},
  pages={51787--51810},
  year={2024}
}

@inproceedings{xu2024energy,
  title={Energy-based concept bottleneck models: Unifying prediction, concept intervention, and probabilistic interpretations},
  author={Xu, Xinyue and Qin, Yi and Mi, Lu and Wang, Hao and Li, Xiaomeng},
  booktitle={International Conference on Learning Representations},
  volume={2024},
  pages={6944--6966},
  year={2024}
}

@inproceedings{almudevar2026there,
  title={There was never a bottleneck in concept bottleneck models},
  author={Almud{\'e}var, Antonio and Hern{\'a}ndez Lobato, Jos{\'e} Miguel and Ortega, Alfonso},
  booktitle={International Conference on Learning Representations},
  volume={2026},
  pages={114146--114165},
  year={2026}
}

@article{yuksekgonul2022post,
  title={Post-hoc concept bottleneck models},
  author={Yuksekgonul, Mert and Wang, Maggie and Zou, James},
  journal={arXiv preprint arXiv:2205.15480},
  year={2022}
}

@article{wah2011caltech, title={The caltech-ucsd birds-200-2011 dataset}, author={Wah, Catherine and Branson, Steve and Welinder, Peter and Perona, Pietro and Belongie, Serge}, year={2011}, publisher={California Institute of Technology} }

@article{espinosa2023learning,
  title={Learning to receive help: Intervention-aware concept embedding models},
  author={Espinosa Zarlenga, Mateo and Collins, Katie and Dvijotham, Krishnamurthy and Weller, Adrian and Shams, Zohreh and Jamnik, Mateja},
  journal={Advances in Neural Information Processing Systems},
  volume={36},
  pages={37849--37875},
  year={2023}
}

@article{mahinpei2021promises,
  title={Promises and pitfalls of black-box concept learning models},
  author={Mahinpei, Anita and Clark, Justin and Lage, Isaac and Doshi-Velez, Finale and Pan, Weiwei},
  journal={arXiv preprint arXiv:2106.13314},
  year={2021}
}

@article{havasi2022addressing,
  title={Addressing leakage in concept bottleneck models},
  author={Havasi, Marton and Parbhoo, Sonali and Doshi-Velez, Finale},
  journal={Advances in Neural Information Processing Systems},
  volume={35},
  pages={23386--23397},
  year={2022}
}

@article{espinosa2022concept,
  title={Concept embedding models: Beyond the accuracy-explainability trade-off},
  author={Espinosa Zarlenga, Mateo and Barbiero, Pietro and Ciravegna, Gabriele and Marra, Giuseppe and Giannini, Francesco and Diligenti, Michelangelo and Shams, Zohreh and Precioso, Frederic and Melacci, Stefano and Weller, Adrian and others},
  journal={Advances in neural information processing systems},
  volume={35},
  pages={21400--21413},
  year={2022}
}

@article{parisini2025leakage,
  title={Leakage and interpretability in concept-based models},
  author={Parisini, Enrico and Chakraborti, Tapabrata and Harbron, Chris and MacArthur, Ben D and Banerji, Christopher RS},
  journal={arXiv preprint arXiv:2504.14094},
  year={2025}
}

@article{margeloiu2021concept,
  title={Do concept bottleneck models learn as intended?},
  author={Margeloiu, Andrei and Ashman, Matthew and Bhatt, Umang and Chen, Yanzhi and Jamnik, Mateja and Weller, Adrian},
  journal={arXiv preprint arXiv:2105.04289},
  year={2021}
}

@inproceedings{shin2023closer,
  title={A closer look at the intervention procedure of concept bottleneck models},
  author={Shin, Sungbin and Jo, Yohan and Ahn, Sungsoo and Lee, Namhoon},
  booktitle={International conference on machine learning},
  pages={31504--31520},
  year={2023},
  organization={PMLR}
}

@article{parisini2026concept,
  title={Concept frustration: Aligning human concepts and machine representations},
  author={Parisini, Enrico and Soelistyo, Christopher J and Isaac, Ahab and Barp, Alessandro and Banerji, Christopher RS},
  journal={arXiv preprint arXiv:2603.29654},
  year={2026}
}

@misc{chauhan_interactive_2023,
	title = {Interactive {Concept} {Bottleneck} {Models}},
	url = {http://arxiv.org/abs/2212.07430},
	doi = {10.48550/arXiv.2212.07430},
	urldate = {2026-09-18},
	publisher = {arXiv},
	author = {Chauhan, Kushal and Tiwari, Rishabh and Freyberg, Jan and Shenoy, Pradeep and Dvijotham, Krishnamurthy},
	month = apr,
	year = {2023},
	note = {arXiv:2212.07430 [cs.LG]},
}

@misc{alvarez-melis_towards_2018,
	title = {Towards {Robust} {Interpretability} with {Self}-{Explaining} {Neural} {Networks}},
	url = {http://arxiv.org/abs/1806.07538},
	doi = {10.48550/arXiv.1806.07538},
	urldate = {2026-09-18},
	publisher = {arXiv},
	author = {Alvarez-Melis, David and Jaakkola, Tommi S.},
	month = dec,
	year = {2018},
	note = {arXiv:1806.07538 [cs.LG]},
}

@article{chen_concept_2020,
	title = {Concept {Whitening} for {Interpretable} {Image} {Recognition}},
	volume = {2},
	issn = {2522-5839},
	url = {http://arxiv.org/abs/2002.01650},
	doi = {10.1038/s42256-020-00265-z},
	number = {12},
	urldate = {2026-09-18},
	journal = {Nature Machine Intelligence},
	author = {Chen, Zhi and Bei, Yijie and Rudin, Cynthia},
	month = dec,
	year = {2020},
	note = {arXiv:2002.01650 [cs.LG]},
	pages = {772--782},
}

@misc{sawada_concept_2022,
	title = {Concept {Bottleneck} {Model} with {Additional} {Unsupervised} {Concepts}},
	url = {http://arxiv.org/abs/2202.01459},
	doi = {10.48550/arXiv.2202.01459},
	urldate = {2026-09-18},
	publisher = {arXiv},
	author = {Sawada, Yoshihide and Nakamura, Keigo},
	month = feb,
	year = {2022},
	note = {arXiv:2202.01459 [cs.CV]},
}

@misc{oikarinen_label-free_2023,
	title = {Label-{Free} {Concept} {Bottleneck} {Models}},
	url = {http://arxiv.org/abs/2304.06129},
	doi = {10.48550/arXiv.2304.06129},
	urldate = {2026-09-18},
	publisher = {arXiv},
	author = {Oikarinen, Tuomas and Das, Subhro and Nguyen, Lam M. and Weng, Tsui-Wei},
	month = jun,
	year = {2023},
	note = {arXiv:2304.06129 [cs.LG]},
}

@misc{kim_probabilistic_2023,
	title = {Probabilistic {Concept} {Bottleneck} {Models}},
	url = {http://arxiv.org/abs/2306.01574},
	doi = {10.48550/arXiv.2306.01574},
	publisher = {arXiv},
	author = {Kim, Eunji and Jung, Dahuin and Park, Sangha and Kim, Siwon and Yoon, Sungroh},
	month = jun,
	year = {2023},
	note = {arXiv:2306.01574 [cs.LG]},
}

@misc{zarlenga_defense_2026,
	title = {In {Defense} of {Information} {Leakage} in {Concept}-based {Models}},
	url = {http://arxiv.org/abs/2606.10669},
	doi = {10.48550/arXiv.2606.10669},
	urldate = {2026-09-18},
	publisher = {arXiv},
	author = {Espinosa Zarlenga, Mateo},
	month = jun,
	year = {2026},
	note = {arXiv:2606.10669 [cs.LG]},
}

@misc{enouen_debugging_2026,
	title = {Debugging {Concept} {Bottleneck} {Models} through {Removal} and {Retraining}},
	url = {http://arxiv.org/abs/2509.21385},
	doi = {10.48550/arXiv.2509.21385},
	urldate = {2026-09-18},
	publisher = {arXiv},
	author = {Enouen, Eric and Galhotra, Sainyam},
	month = mar,
	year = {2026},
	note = {arXiv:2509.21385 [cs.CV]},
}

@misc{makonnen_measuring_2025,
	title = {Measuring {Leakage} in {Concept}-{Based} {Methods}: {An} {Information} {Theoretic} {Approach}},
	shorttitle = {Measuring {Leakage} in {Concept}-{Based} {Methods}},
	url = {http://arxiv.org/abs/2504.09459},
	doi = {10.48550/arXiv.2504.09459},
	urldate = {2026-09-18},
	publisher = {arXiv},
	author = {Makonnen, Mikael and Vandenhirtz, Moritz and Laguna, Sonia and Vogt, Julia E.},
	month = jul,
	year = {2025},
	note = {arXiv:2504.09459 [cs.LG]},
}
